\documentclass{article} % For LaTeX2e
\usepackage{arxiv}

\usepackage[utf8]{inputenc} % allow utf-8 input
\usepackage[T1]{fontenc} 
\usepackage[numbers]{natbib}  % or [authoryear]

\usepackage{hyperref}       % hyperlinks
\usepackage{url}            % simple URL typesetting
\usepackage{booktabs}       % professional-quality tables
\usepackage{amsfonts}       % blackboard math symbols
\usepackage{nicefrac}       % compact symbols for 1/2, etc.
\usepackage{microtype}      % microtypography
\usepackage{amsthm}
\usepackage{enumitem}
\usepackage{graphicx}
\usepackage{subfigure}
\usepackage{xspace}
\usepackage{wrapfig}
\usepackage[table]{xcolor}
\usepackage{algorithm}
\usepackage{algpseudocode}
\usepackage{amsmath, amssymb}
\usepackage{cleveref}

\usepackage[skip=2pt]{caption}

\usepackage{titlesec}
\titlespacing{\section}{0pt}{6pt plus 2pt}{4pt}
\titlespacing{\subsection}{0pt}{4pt plus 1pt}{3pt}

\newtheorem{theorem}{Theorem}[section]
\newtheorem{proposition}[theorem]{Proposition}

\newtheorem{corollary}[theorem]{Corollary}
\theoremstyle{definition}
\newtheorem{definition}[theorem]{Definition}

\theoremstyle{remark}

\usepackage{multirow}
\definecolor{goodcell}{HTML}{bfd9c1}
\definecolor{badcell}{HTML}{e3a6a6}
\definecolor{first}{HTML}{00A64F}  % green
\definecolor{second}{HTML}{006EB8} % blue
\newcommand{\one}[1]{\textcolor{first}{\bf#1}}
\newcommand{\two}[1]{\textcolor{second}{\bf#1}}
\newcommand{\three}[1]{{\bf#1}}

\usepackage{amsmath,amsfonts,bm}

\def\eqref#1{equation~\ref{#1}}
\def\1{\bm{1}}

\def\vb{{\bm{b}}}

\def\ve{{\bm{e}}}

\def\vh{{\bm{h}}}

\def\vj{{\bm{j}}}

\def\vu{{\bm{u}}}
\def\vv{{\bm{v}}}

\def\vx{{\bm{x}}}

\def\valpha{{\bm{\alpha}}}

\def\mA{{\bm{A}}}

\def\mW{{\bm{W}}}
\def\mX{{\bm{X}}}
\def\mY{{\bm{Y}}}

\DeclareMathAlphabet{\mathsfit}{\encodingdefault}{\sfdefault}{m}{sl}
\SetMathAlphabet{\mathsfit}{bold}{\encodingdefault}{\sfdefault}{bx}{n}
\newcommand{\tens}[1]{\bm{\mathsfit{#1}}}

\def\tD{{\tens{D}}}

\def\tH{{\tens{H}}}

\def\gG{{\mathcal{G}}}

\def\gM{{\mathcal{M}}}
\def\gN{{\mathcal{N}}}

\def\gS{{\mathcal{S}}}
\def\gT{{\mathcal{T}}}

\def\sN{{\mathbb{N}}}

\def\sR{{\mathbb{R}}}

\newcommand{\mlp}{\mathsf{MLP}}
\newcommand{\agg}{\text{AGG}}
\newcommand{\ourgnn}{\Phi}
\newcommand{\osan}{\Psi}

\newcommand{\mpnn}{\gM}
\newcommand{\gnn}{\gT}
\newcommand{\pool}{\text{U}}

\newcommand{\ourmethod}{HOD-GNN\xspace} % Yoav: Added xspace to make it more readable

\definecolor{fabgreen}{rgb}{0.0, 0.5, 0.37}

\title{On The Expressive Power of GNN
Derivatives}

\author{Yam Eitan\textsuperscript{1}, Moshe Eliasof\textsuperscript{2}, Yoav Gelberg\textsuperscript{3}, Fabrizio Frasca\textsuperscript{1}, Guy Bar-Shalom\textsuperscript{1}, Haggai Maron\textsuperscript{1,4} \\
\textsuperscript{1} Technion – Israel Institute of Technology  \\
\textsuperscript{2} University of Cambridge \\
\textsuperscript{3} University of Oxford \\
\textsuperscript{4} NVIDIA Research}

\begin{document}

\maketitle

\begin{abstract}
Despite significant advances in Graph Neural Networks (GNNs), their limited expressivity remains a fundamental challenge. Research on GNN expressivity has produced many expressive architectures, leading to architecture hierarchies with models of increasing expressive power. Separately, derivatives of GNNs with respect to node features have been widely studied in the context of the oversquashing and over-smoothing phenomena, GNN explainability, and more. To date, these derivatives remain unexplored as a means to enhance GNN expressivity. In this paper, we show that these derivatives provide a natural way to enhance the expressivity of GNNs. We introduce High-Order Derivative GNN (\ourmethod), a novel method that enhances the expressivity of Message Passing Neural Networks (MPNNs) by leveraging high-order node derivatives of the base model. These derivatives generate expressive structure-aware node embeddings processed by a second GNN in an end-to-end trainable architecture. Theoretically, we show that the resulting architecture family's expressive power aligns with the WL hierarchy. We also draw deep connections between \ourmethod, Subgraph GNNs, and popular structural encoding schemes. For computational efficiency, we develop a message-passing algorithm for computing high-order derivatives of MPNNs that exploits graph sparsity and parallelism. Evaluations on multiple graph learning benchmarks demonstrate \ourmethod’s excellent performance on popular graph learning tasks.
\end{abstract}

\section{Introduction}
\label{sec:intro}

\looseness=-1
Graph Neural Networks (GNNs) have become foundational tools in geometric deep learning, with widespread applications in domains such as life sciences \citep{wong2024discovery}, social sciences \citep{monti2019fake}, optimization \citep{cappart2023combinatorial}, and more. Despite their empirical success, many GNNs suffer from a fundamental limitation: their expressive power is inherently bounded. In particular, the widely used family of Message Passing Neural Networks (MPNNs) is at most as expressive as the Weisfeiler–Lehman (1-WL) graph isomorphism test \cite{morris2019weisfeiler,xu2018powerful}, limiting their ability to distinguish between even simple non-isomorphic graphs and capture intricate structural patterns \cite{chen2020can}. To address this shortcoming, a growing body of work has proposed more expressive GNN architectures, typically organized into expressivity hierarchies that balance computational cost with representational power \cite{maron2019provably,morris2019weisfeiler,morris2023weisfeiler}. 

\looseness=-1
Concurrently to advances in GNN expressivity, the derivatives of the final node representations $\vh_v^{(T)}$ \footnote{$\vh_v^{(t)}$ is the representation of the node $v$ after the $t$-th GNN layer.} and the graph-level output $\vh^{\text{out}}$ with respect to the initial features $\mX_v$ have played a key role in several research directions. For over-squashing analysis \citep{di2023over, didoes}, both first-order derivatives $\frac{\partial \vh^{(T)}_v}{\partial \mX_u}$ and mixed partial derivatives $\frac{\partial^2 \vh^{\text{out}}}{\partial \mX_v \partial \mX_u}$ quantify inter-node influence and communication capacity. In over-smoothing studies like \citet{arroyo2025vanishing}, derivatives $\frac{\partial \vh^{\text{out}}}{\partial \mX_v}$ are used to analyze vanishing gradients, connecting over-smoothing to diminished gradient flow. GNN gradient-based explainability methods \citep{baldassarre2019explainability, pope2019explainability} also use these derivatives to identify influential nodes and features. The prevalence of these derivatives across diverse contexts in GNN research suggests they encode valuable information that may be informative for graph learning tasks.

\looseness=-1
\textbf{Our approach.} In this work, we reveal a surprising connection between these two lines of research. We show that incorporating derivatives of a base MPNN with respect to initial node features as additional inputs to a downstream MPNN enhances the expressivity of the base components. One intuitive way to understand this connection is through the mechanism by which GNNs with marking\citep{papp2022theoretical,pellizzoni2024expressivity} improve expressivity: they choose a node from the input graph and add to it a unique identifier before processing it through an MPNN. While these identifiers are often implemented through an explicit, often discrete, perturbations to the node features, our approach instead computes derivatives of the MPNN output, capturing the effect of \textit{infinitesimal} perturbations. This provides equivalent expressivity gains while also introducing an inductive bias toward derivative-aware representations. See Section~\ref{sec:method}  for a detailed discussion.

% One intuitive way to understand this connection is through the mechanism by which Subgraph GNNs \citep{bevilacqua2021equivariant, cotta2021reconstruction,papp2022theoretical} improve expressivity: they process multiple copies of the input graph, each augmented with node markings-feature vectors that distinguish specific nodes in the graph. While Subgraph GNNs implement their markings through an explicit, often discrete, perturbations to the node features, our approach instead computes derivatives of the MPNN output, capturing the effect of \textit{infinitesimal} perturbations. This provides equivalent expressivity gains while also introducing an inductive bias toward derivative-aware representations. 

We introduce High-Order Derivative GNN (\ourmethod), a novel expressive GNN family that leverages the derivatives of a base MPNN to improve its expressive power. We first introduce $1$-\ourmethod, which consists of three components: a base MPNN, a derivative encoder network, and a downstream GNN. $1$-\ourmethod\ computes high-order derivatives of the base MPNN with respect to the features of a \emph{single} node at a time, i.e., $\frac{\partial^\alpha \vh^{(T)}_v}{\partial \mX_u^\alpha}$ and $\frac{\partial^\alpha \vh^{\text{out}}}{\partial \mX_u^\alpha}$. These derivatives are then encoded into new derivative-aware node features via the encoder network, which are then passed to the downstream GNN. Theoretically,  We show that $1$-\ourmethod\ models are more expressive than standard GNNs, can compute popular structural encodings, and are tightly related to Subgraph GNNs \cite{cotta2021reconstruction, bevilacqua2021equivariant}. Empirically, we demonstrate several desirable properties of our model: it achieves strong performance across a range of standard graph benchmarks, scales to larger graphs that remain out of reach for other expressive GNNs, and can accurately count graph substructures, providing direct empirical evidence of its expressive power.

% it achieves strong performance across a range of standard graph benchmarks, delivers significant gains even when using base MPNNs with very small hidden dimensions (as low as 8), and remains effective when using very deep base MPNNs (up to 20 layers).

\looseness=-1
We then extend $1$-\ourmethod\ to $k$-\ourmethod, which supports mixed derivatives with respect to $k$ distinct node features (i.e. $\frac{\partial^{\alpha_1+\dots+\alpha_k} \vh^{(T)}_v}{\partial \mX_{u_1}^{\alpha_1},\dots,\mX_{u_k}^{\alpha_k} }$, $\frac{\partial^{\alpha_1+\dots+\alpha_k} \vh^{\text{out}}}{\partial \mX_{u_1}^{\alpha_1},\dots,\mX_{u_k}^{\alpha_k} }$). Like $1$-\ourmethod, the $k$-\ourmethod\ forward pass begins by computing higher-order mixed derivatives of a base MPNN, which form a $k$-indexed derivative tensor. This tensor is then used to construct new node features using a higher-order encoder network (as in \citet{maron2018invariant, morris2019weisfeiler}), which are subsequently passed to a downstream GNN for final prediction. We theoretically analyze $k$-\ourmethod, showing that it can distinguish between graphs that are indistinguishable to the $k$-WL test\footnote{We refer here to the folklore WL test rather than the oblivious variant; see \citet{morris2023weisfeiler} for a detailed discussion of the differences.}, resulting in a model that is more expressive than any of its individual components alone. Furthermore, we leverage results from \citet{zhang2024beyond} to analyze $k$-\ourmethod's ability to compute homomorphism counts, demonstrating its capacity to capture fine-grained structure.

% Furthermore, we leverage results from \citet{zhang2024beyond} to show that $k$-\ourmethod\ can compute homomorphism counts for $k$-apex forests—graphs that become trees upon removal of $k$ nodes—demonstrating its capacity to capture fine-grained structure.

\looseness=-1
Efficiently computing high-order node derivatives is a core component of \ourmethod. To this end, we develop a novel algorithm for computing these derivatives via an analytic, message-passing-like procedure. This approach yields two key benefits. First, being fully analytic, it enables differentiation through the derivative computation itself, allowing \ourmethod\ to be trained end-to-end. Second, the message-passing-like structure exploits the sparsity of graph data, improving scalability (see Section~\ref{sec:theory} for a detailed complexity analysis). Combined with the empirical observation that \ourmethod\ remains effective even when using base MPNNs with small hidden dimensions, our method scales to benchmarks containing larger graphs that are often out of reach for other expressive GNN architectures.

\textbf{Our contributions.}  
(1) We introduce k-\ourmethod, a novel expressive GNN family that integrates derivative-based embeddings;  
(2) We provide a theoretical analysis of its expressivity and computational properties;  
(3) We propose an algorithm for efficient derivative computation on graphs;  
(4) We demonstrate consistently high empirical performance across seven standard graph classification and regression benchmarks. Additionally, we show that \ourmethod\ scales to benchmarks containing larger graphs that are typically out of reach for many expressive architectures on standard hardware.

\section{Preliminaries and Previous Work}
\label{sec:preliminaries}

% \looseness=-1
\textbf{Notation.} 
 The size of a set $\gS$ is denoted by $|\gS|$. $\oplus$ denotes concatenation. We denote graphs by $\gG = (\mA, \mX)$, where $\mA \in \sR^{n \times n}$ is the adjacency matrix and $\mX \in \sR^{n \times d}$ is the node feature matrix, with $n$ nodes and $d$-dimensional features per node.  The node set of a graph is denoted by $V(\gG)$.

% The set of $k$-tuples  of nodes of a graph is denoted by $V^k(\gG)$.

\textbf{MPNNs and GNN expressivity.} 
MPNNs\citep{gilmer2017neural} are a widely used class of GNNs that update node representations through iterative aggregation of local neighborhood information. At each layer $t$, the representation $\vh_v^{(t)}$ of node $v$ is updated via:
\begin{equation}
\label{eq:mpnn}
 \vh_v^{(t)} = \mlp^{(t)}\left(\vh_v^{(t-1)}, \agg^{(t)}\left(\left\{\vh_u^{(t-1)} : u \in \gN(v)\right\}\right)\right),   
\end{equation}

where $\gN(v)$ denotes the neighbors of node $v$ in the graph and $\agg^{(t)}$ are permutation-preserving aggregation function. After $T$ message-passing layers, a graph-level representation is typically obtained by applying a global pooling operation over all node embeddings:
\begin{equation}
\label{eq:mpnn_pooling}
    \vh^{\text{out}} = \agg^{\text{out}}\left(\left\{\vh_v^{(T)} \mid v \in V(\gG)\right\}\right),
\end{equation}

MPNNs have inherent expressivity limitations~\citep{morris2019weisfeiler, xu2018powerful, weisfeiler1968reduction}, as they cannot distinguish graphs that are indistinguishable by the 1-WL test. To address this, a wide range of more expressive GNN architectures have been proposed~\citep{morris2023weisfeiler, maron2018invariant, puny2023equivariant, cotta2021reconstruction, rieck2019persistent, sato2021random, dwivedi2023benchmarking}. (See Appendix~\ref{apx:prev_work} for details or \citep{sato2020survey, morris2023weisfeiler, jegelka2022theory, li2022expressive, zhang2024expressive-b} for comprehensive surveys.)

% In this work, we focus on one such family-Subgraph GNNs-which serve as a foundation for our analysis of \ourmethod.

\looseness=-1
\textbf{Subgraph GNNs.}
 Subgraph GNNs \citep{zhang2021nested, cotta2021reconstruction, bevilacqua2021equivariant, frasca2022understanding,zhang2023rethinking,zhang2023complete,bar2024flexible} are expressive GNNs that operate over a set of subgraphs $\mathcal{B}_\gG = \{\gS_\vv \mid \vv \in V^k(\gG) \}$, where each subgraph $\gS_\vv$ corresponds to a node or $k$-tuple of nodes from the input graph $\gG$. In this work, we focus on the widely adopted node-marking DS-GNNs \citep{cotta2021reconstruction,bevilacqua2021equivariant,papp2022theoretical} and their higher-order generalization, $k$-OSAN \citep{qian2022ordered}, though we note that many other variants of Subgraph GNNs exist. For precise defintions of DS-GNN and k-OSAN, see Appendix \ref{apx:definitions}.

% \looseness=-1
% A DS-GNN model $\osan$ consists of a base MPNN $\mpnn$ and a downstream MPNN $\gnn$. Given an input graph $\gG = (\mA, \mX)$, each subgraph $\gS_v$ in the bag $\gB_\gG$ is a copy of $\gG$ in which the feature of node $v$ is augmented with a marking vector—a one-hot vector $\ve_v$ that distinguishes it from the others. That is, $\gS_v = (\mA, \mX \oplus \ve_v)$. The base MPNN is then applied independently to each subgraph, producing a new node feature for each corresponding node: $\vh_{v}^{\text{sub}} = \mpnn(\gS_v)$ according to Equations \ref{eq:mpnn} and \ref{eq:mpnn_pooling}. The resulting node feature matrix $\vh^{\text{sub}}$ is then passed to the second MPNN $\gnn$ to generate a graph-level prediction. The $k$-OSAN framework generalizes this idea to subgraphs indexed by $k$-tuples of nodes;  a formal definition of $k$-OSAN is provided in Appendix~\ref{apx:definitions}. $k$-OSAN networks have been shown to be highly expressive: they can distinguish between graphs that are indistinguishable by the $k$-WL test \citep{qian2022ordered}, and can compute homomorphism counts for $k$-apex forests \citep{zhang2024beyond}.

\looseness=-1
\textbf{Derivatives of MPNNs.} 
Derivatives frequently appear in the analysis of GNNs. In the study of \textit{oversquashing}---the failure of information to propagate through graph structures~\citep{alon2020bottleneck, topping2021understanding, di2023over, didoes}---derivatives play a key role (For a comprehensive overview, see ~\citet{akansha2023over}). Node derivatives are also used in GNN explainability~\citep{ying2019gnnexplainer,luo2020parameterized,baldassarre2019explainability,pope2019explainability}. Gradient-based approaches such as Sensitivity Analysis, Guided Backpropagation~\citep{baldassarre2019explainability}, and Grad-CAM~\citep{pope2019explainability} rely on derivative magnitudes. Finally, several standalone works make use of node-based derivatives. E.g., \citet{arroyo2025vanishing} use node derivatives to draw a connection between vanishing gradients, and over-smoothing, and \citet{keren2024sequential} propose aggregation functions, designed to induce non-zero mixed node derivatives. See Appendix~\ref{apx:prev_work} for further discussion.
\vspace{-10pt}

\section{Method}
\label{sec:method}
% \vspace{-1pt}

We begin this section with a discussion motivating the use of MPNN derivatives and their contribution to improving expressivity. We then introduce $k$-\ourmethod, an expressive GNN architecture that enhances representational power by leveraging derivatives of a base MPNN. We first present the full details of the $1$-\ourmethod\ model, followed by an overview of its higher-order generalization. A comprehensive treatment of the general $k$-order case is provided in Appendix~\ref{apx:method}.  We emphasize that in $k$-\ourmethod, the parameter $k$ refers to the number of distinct nodes with respect to which derivatives are taken, not the total derivative order. For instance, $1$-\ourmethod\ uses derivatives of the form $\frac{\partial^\alpha \vh}{\partial^\alpha X_v}$, but not mixed derivatives such as $\frac{\partial^{\alpha_1 + \alpha_2} \vh}{\partial^{\alpha_1} X_v, \partial^{\alpha_2} X_u}$, which involve multiple nodes.

\subsection{Motivation}
\label{sec:motivation}

\looseness=-1
Beyond being a widely used and informative quantity in GNN analysis, MPNN derivatives can enhance expressivity. To build intuition to why that is, we begin with a simple example, showing that first-order derivatives allow us to count triangles, a task that standard MPNNs cannot perform. Consider the model $\mpnn(\mA, \mX) = \mA^3 \mX$, which can be implemented by a three-layer GCN with identity activation. For any node $v$, the derivative of its final feature vector $\vh_v$ with respect to its own input feature vector $\mX_v$ is exactly $\mA^3_{v,v}$. aggregating these derivatives, we can compute $\sum_v \mA^3_{v,v}/6$, which is exactly the number of triangles in the graph. 

\looseness=-1  
To illustrate how higher-order derivatives further enhance expressivity, we recall that GNNs with marking \citep{papp2022theoretical} improve expressive power over standard MPNNs by selecting a node \footnote{While GNNs with marking can select multiple nodes, for clarity we focus on the single-node case. The general case is discussed in Appendix~\ref{apx:motivation}.}
$v$ in the input graph $\gG = (\mA, \mX)$ and attaching a unique identifier to it, yielding the modified input $\mX + \epsilon \ve_v$ for some $\epsilon \in \sR$. The output is then $\vh^{*} = \mpnn(\mA, \mX + \epsilon \ve_v)$.   If $\mpnn$ employs an analytic activation  function $\sigma$ (See definition \ref{def:analytic_function} in the appendix), then $\mpnn$ itself is analytic. Consequently, its output can be approximated by the Taylor expansion:  
\begin{equation}
     \mpnn(\mA, \mX + \epsilon \ve_v) \approx \sum_{i=0}^m 
     \frac{\partial^i \mpnn(\mA, \mX + x \ve_v)}{\partial x^i}\Big|_{x=0}  \cdot \epsilon^i 
     = \sum_{i=0}^m \frac{\partial^i \vh^{\text{out}}}{\partial \ve_v^i} \cdot \epsilon^i, 
\end{equation}  
where $\vh^{\text{out}} = \mpnn(\mA, \mX)$ is the output of $\mpnn$ without marking.  This shows that by leveraging the higher-order derivatives of an MPNN, one can approximate the output of a GNN with marking to arbitrary precision. As a result, derivatives strictly extend the expressive power of MPNNs. An expanded intuitive discussion of these expressivity gains, along with the natural connection between \ourmethod and subgraph GNNs, is provided in Appendix~\ref{apx:motivation}.

\subsection{The 1-\ourmethod\ architecture}
\label{sec:1-ourmethod}

\begin{figure}[t]
    \centering
    \vspace{-36pt}    \includegraphics[width=0.85\textwidth, trim=0 50 0 50, clip]{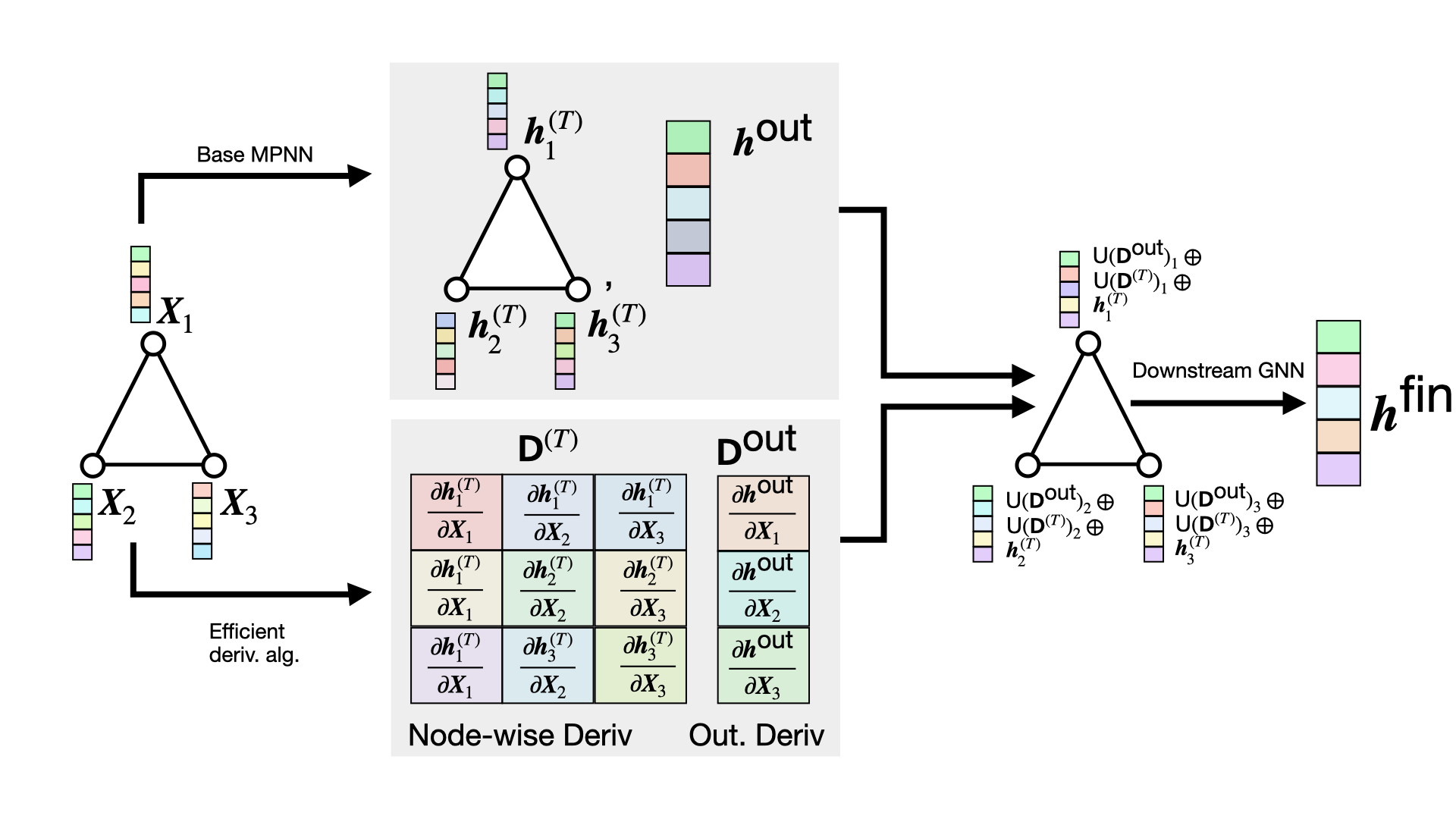}
    % \vspace{-10pt}
    \caption{The \ourmethod\ pipeline. Given an input graph, we compute the outputs and derivatives of a base MPNN. The derivatives are processed by two encoders (denoted $\pool$) to produce features that are concatenated with the base  MPNN outputs and passed to a downstream GNN for final prediction.}\label{fig:ourmethod}
    % \vspace{-10pt}
\end{figure}

\looseness=-1
\textbf{Overview.}  A $1$-\ourmethod~ model, denoted $\ourgnn$, consists of two GNNs: a base MPNN $\mpnn$ and a downstream network $\gnn$ as well as two derivative encoder networks $\pool^{\text{node}}$ and $\pool^{\text{out}}$. Given an input graph $\gG$, the computation of $\ourgnn(\gG)$ proceeds in four steps (see Figure \ref{fig:ourmethod}):  (1) Compute the final node representations and output of $\mpnn$; (2) Compute derivative tensors of the output with respect to the input node features (defined below); (3) Use $\pool^{\text{node}} $ and $\pool^{\text{out}}$ to extract new derivative informed node features from the derivative tensors. (4) Apply $\gnn$ to the derivative-informed features. Importantly, we develop an efficient algorithm for step (2) that enables backpropagation through the derivative computation itself, making all four above steps differentiable (see Section~\ref{sec:gradient_computation}). Consequently, the entire \ourmethod\ model can be trained end-to-end, a strategy we adopt in all our experiments.

\looseness=-1
\textbf{Steps 1 \& 2.} In the first two stages, we compute the final node representations $\vh^{(T)}$ and the output vector $\vh^{\text{out}}$ using the base MPNN $\mpnn$, along with their corresponding derivative tensors, defined below:
\begin{definition}
    \looseness=-1
    Given a graph $\gG=(\mA, \mX)$ with $n$ nodes, an MPNN $\mpnn$ and an intermediate node feature representation matrix $\vh \in \sR^{n \times d'}$, the derivative tensor of $\tD(\vh) \in \sR^{n\times n \times  d'\times d \times m} $
    is defined by:
    \begin{equation}
         \tD(\vh)[v, u, i, j , \alpha ] = \frac{\partial^{\alpha} \vh_{v,i}}{\partial \mX^{\alpha}_{u,j}},
    \end{equation}
    where $v,u \in V(\gG)$ are nodes, $i \in [d'], j \in [d]$ specify the feature dimensions of the node feature vectors $\vh_v, \mX_u$ respectively, and $\alpha \in [m]$ where  $m \in \sN$ is a hyperparameter specifying the maximum order of derivatives to be considered. 
    Similarly, given a graph-level prediction vector $\vh^\text{out} \in \sR^{d'}$ the derivative tensor $\tD(\vh^\text{out}) \in \sR^{  n \times d' \times d \times m}$ is defined by:
    \begin{equation}
         \tD(\vh^\text{out})[ u,i, j ,\alpha ] = \frac{\partial^{\alpha} \vh^\text{out}_{i}}{\partial \mX^{\alpha}_{u,j}}.
    \end{equation}
\end{definition}
In 1-\ourmethod, we compute the output derivative tensor $\tD^{\text{out}} = \tD(\vh^{\text{out}})$, which captures how the output of the base MPNN $\mpnn$ responds to perturbations in the input node features. In parallel, we compute the node-wise derivative tensor $\tD^{(T)}$, where $\tD^{(t)} = \tD(\vh^{(t)})$ for $t = 1, \dots, T$. These tensors characterize how each node’s representation at layer $t$ changes in response to variations in the input features. Derivative tensors are computed using Algorithm~\ref{alg:derivative_computation}, described in Section~\ref{sec:gradient_computation} and elaborated on in Appendix~\ref{apx:derivative_comutation}. The algorithm leverages the sparsity of the input graph to enable efficient computation of high-order derivatives. Crucially, Algorithm~\ref{alg:derivative_computation} is fully differentiable with respect to the weights of $\mpnn$, enabling end-to-end training of $\ourgnn$.

\textbf{Step 3.}  
In the third stage of our method, we extract new node features from the derivative tensors $\tD^{\text{out}}$ and $\tD^{(T)}$ using the encoder networks $\pool^{\text{node}}$ and $\pool^{\text{out}}$.  
First, as $\tD^{\text{out}} \in \mathbb{R}^{ n  \times d' \times d \times m}$ is a tensor indexed by a single node, it can be directly interpreted as a node feature matrix by flattening the remaining dimensions. We thus define the encoder network $\pool^{\text{out}}$ to be a DeepSets \citep{zaheer2017deep} update:
 \begin{equation}
     \pool^{\text{out}}(\tD^{\text{out}})_v = \mlp(\tD^{\text{out}}[v,\dots]) .
 \end{equation}
 
\looseness=-1
Secondly, since $\tD^{(T)}$ is indexed by pairs of nodes in $\gG$, it shares the structure of the adjacency matrix $\mA$, which is also pairwise-indexed. We can thus define the encoder network $\pool^{\text{node}}: \mathbb{R}^{n^2 \times d' \times d \times m} \to \mathbb{R}^{n \times \bar{d}}$ to be any GNN architecture which maps adjacency matrices with edge features to node feature matrices. To enhance sensitivity to global interactions, we select  $\pool^{\text{node}}$ to be a 2-Invariant Graph Network (IGN)\footnote{We can also employ a modified version of the standard 2-IGN that preserves the sparsity structure of $\tD^{(T)}$. See Appendix~\ref{apx:method} for details.}\citep{maron2018invariant}.

We construct the derivative-informed node features $\vh^{\text{der}}$ by combining information from the base MPNN $\mpnn$, the pooled intermediate derivatives, and the output derivatives:
\begin{equation}
\label{eq:derivative_informed_node_features}
    \vh^{\text{der}}_v =  \vh^{(T)}_v \oplus \pool^{\text{out}}(\tD^{\text{out}}) \oplus \pool^{\text{node}}(\tD^{\text{(T)}}).
\end{equation}

\textbf{Step 4.}  
In the final stage, we replace the original node features of $\gG$ with the derivative-informed features $\vh^{\text{der}}$, and apply a downstream GNN $\gnn$ to produce a graph-level prediction. For the remainder of this work, we assume $\gnn$ is an MPNN, though our approach is compatible with any GNN architecture.

\subsubsection{Efficient derivative tensor computation}
\label{sec:gradient_computation}

We now describe an efficient algorithm for computing the derivative tensors in a 1-\ourmethod~ model $\ourgnn$ with base MPNN $\mpnn$. For clarity, we focus on the case where $\mpnn$ is a GIN \citep{xu2018powerful}, in which case the message-passing and readout functions are given by:
\begin{equation}
\label{eq:gin}
\vh_v^{(t)} = \mlp^{(t)}\big((1 + \epsilon)\vh_v^{(t-1)} + \sum_{u \in \gN(v)} \vh_u^{(t-1)} \big), \quad 
\vh^{\text{out}} = \mlp \big( \sum_{u \in V(\gG)} \vh_u^{(T)} \big).
\end{equation}
An extension of this algorithm to general MPNNs and higher-order mixed derivatives is provided in Appendix~\ref{apx:derivative_comutation}. For convenience, we decompose the node update in Equation~\ref{eq:gin} into two parts: an aggregation step and a DeepSets-based update, given respectively by:
\begin{equation}
\label{eq:gin_combined}
\overbrace{\tilde{\vh}^{(t-1)}_v = (1 + \epsilon)\vh^{(t-1)}_v + \sum_{u \in \gN(v)} \vh^{(t-1)}_u}^{\text{agg. update}}, \quad 
\overbrace{\vh^{(t)}_v = \mlp(\tilde{\vh}^{(t-1)}_v)}^{\text{DeepSets update}}.
\end{equation}
The algorithm is based on the following two observations: First since $\tilde{\vh}^{(t-1)}_v$ is a linear combination of $\vh^{(t-1)}_v$ and its neighboring node features $\{ \vh^{(t-1)}_u \mid u \in \gN(v) \}$, the derivatives of $\tilde{\vh}^{(t-1)}_v$ are likewise linear combinations of the derivatives of $\vh^{(t-1)}_v$ and $\{ \vh^{(t-1)}_u \mid u \in \gN(v) \}$. More explicitly:
\begin{equation}
\label{eq:derivative_gin_agg}
    \tD(\tilde{\vh}^{(t-1)})[v, \dots] = (1 + \epsilon)\tD^{(t-1)}[v, \dots] + \sum_{u \in \gN(v)} \tD^{(t-1)}[u, \dots].
\end{equation}
This computation mirrors the GIN aggregation update in Equation~\ref{eq:gin}, leveraging the sparsity of the graph. Second, since the DeepSets update applies an MLP independently to each node feature $\tilde{\vh}^{(t-1)}_v$, we can apply Faà di Bruno’s  formula (see e.g. \citep{hardy2006combinatorics}) to compute the derivatives of each $\vh^{(t)}_v$ based on the derivatives of $\tilde{\vh}^{(t-1)}_v$. This results in a “DeepSets-like” derivative update, allowing us to compute $\tD^{(t)}$ directly from $\tD(\tilde{\vh}^{(t-1)})$. Iteratively applying these two steps yields the final node-wise derivative tensor $\tD^{(T)}$ through a differentiable, message-passing-like procedure. A similar approach allows efficient computation of $\tD^{\text{out}}$ from $\tD^{(T)}$. See Appendix~\ref{apx:derivative_comutation} for full details of the algorithm. 

\textbf{Computational Complexity.} An important property of the above algorithm is that, for sparse graphs or relatively shallow base MPNNs, it is computationally efficient.  To see this, first notice that since $\vh^{(0)} = \mX$, the tensor $\tD^{(0)}$ is extremely sparse, satisfying:
\begin{equation}
\label{eq:derivative_init}
\tD^{(0)}[v,u,i, j, \alpha] =
\begin{cases}
1 & \text{if } v= u,\ i = j,\ \alpha = 1, \\
0 & \text{otherwise}.
\end{cases}
\end{equation}
Thus, it can be stored efficiently using sparse matrices. At each layer $t$, the derivative aggregation step (Equation~\ref{eq:derivative_gin_agg}) increases the number of non-zero entries only in proportion to the number of node pairs that exchange messages for the first time. Thus, for small values of $t$ or for sparse graphs $\gG$, the derivative tensor $\tD^{(t)}$  remains sparse. Moreover, the algorithm's message-passing-like structure ensures runtime efficiency as well. For a full complexity analysis of our algorithm, see Section \ref{sec:theory}.

\subsection{k-\ourmethod via mixed derivatives}
\label{sec:higher_order_ourmethod}

We now generalize $1$-\ourmethod, which operates on single-node derivatives (i.e., derivatives of the form $\frac{\partial^\alpha \vh^{(T)}_v}{\partial \mX_u^\alpha}$ or $\frac{\partial^\alpha \vh^{\text{out}}}{\partial \mX_u^\alpha}$), to $k$-\ourmethod, which extracts information from mixed partial derivatives across $k$ nodes (i.e., $\frac{\partial^{\alpha_1+\dots+\alpha_k} \vh^{(T)}_v}{\partial \mX{u_1}^{\alpha_1} \cdots \partial \mX{u_k}^{\alpha_k}}$, or$\frac{\partial^{\alpha_1+\dots+\alpha_k} \vh^{\text{out}}}{\partial \mX{u_1}^{\alpha_1} \cdots \partial \mX{u_k}^{\alpha_k}}$). $k$-\ourmethod\ offers increased expressive power at the cost of greater computational complexity. We begin by formally defining the $k$-indexed derivative tensors. For simplicity, we assume the input node features are 1-dimensional, handling the more general case in Appendix \ref{apx:method}.
\begin{definition}
    Given a graph $\gG=(\mA, \mX)$ with $n$ nodes, and an MPNN $\mpnn$ and an intermediate node feature matrix $\vh \in \sR^{n \times d'}$, the $k$-indexed derivative tensor of $\tD_k(\vh) \in \sR^{n \times n^{k}\times d' \times m^k} $
    is defined by:
    \begin{equation}
         \tD_k(\vh)[v, \vu, i, \valpha] = \frac{\partial^{\alpha_1 + \cdots + \alpha_k} \vh_{v,i}}{\partial \mX^{\alpha_1}_{u_1} \cdots \partial\mX^{\alpha_k}_{u_k}}.
    \end{equation}
    where $v \in V(\gG)$, $\vu=(u_1, \dots, u_k) \in V^k(\gG)$ , $i \in [d']$ and $\valpha = (\alpha_1, \dots, \alpha_k) \in [m]^k$. $\tD_k$  is defined similarly for graph-level prediction vectors.
\end{definition}
The $k$-indexed derivative tensors capture how the output and node representations of $\mpnn$ change under joint perturbations to the features of $k$ nodes, thereby encoding rich higher-order structural interactions within the graph. We compute these tensors using an extension of the derivative computation process described in Section \ref{sec:gradient_computation} (See Appendix \ref{apx:derivative_comutation} for more details).

As in the $1$-\ourmethod\ case, a $k$-\ourmethod\ model $\ourgnn$ consists of a base MPNN $\mpnn$ and a downstream network $\gnn$. Given an input graph $\gG$, the computation of $\ourgnn(\gG)$ proceeds in the same four stages established earlier. First, we compute the output and final node representations of $\mpnn$ along with the $k$-indexed derivative tensors $\tD_k^{(T)}$ and $\tD_k^{\text{out}}$. We then use $\vh^{(T)}$, $\tD_k^{(T)}$ and $\tD_k^{\text{out}}$ to extract new derivative informed node features. These node features are computed through:
\begin{equation}
    \vh^{\text{der}}_v = \mlp \left( \vh^{(T)}_v \oplus \pool^{\text{node}}(\tD_k^{(T)})_v \oplus \pool^{\text{out}}(\tD_k^{\text{out}})_v \right).
\end{equation}
where $\pool^{\text{node}}$  and $\pool^{\text{out}}$ are learned $(k+1)$-IGN and $k$-IGN encoders respectively.
Finally, we substitute the original node features of graph $\gG$ with $\vh^{\text{der}}$, and input the resulting graph into $\gnn$ to generate the final graph-level prediction. For more details, see Appendix \ref{apx:method}.

\section{Theoretical Analysis}
\label{sec:theory}
In this section, we analyze the expressive power and computational complexity of $k$-\ourmethod. Formal statements and complete proofs of all results in this section are provided in Appendix~\ref{apx:proofs}.

\textbf{Expressive power.} 
\looseness=-1
We begin by formally relating $k$-\ourmethod  to $k$-OSAN subgraph GNNs, revealing new insights into HOD-GNNs’ expressive power, their position in the WL hierarchy, and their ability to compute complex homomorphism counts.

% Section~\ref{sec:motivation} introduced an intuitive connection between Subgraph GNNs and $1$-\ourmethod. The next theorem makes this connection precise and extends it to the higher-order case.

\begin{theorem}[informal]
\label{thm:ourmethod_vs_osan}
Any  $k$-OSAN model  can be approximated by a $k$- \ourmethod model using an analytic activation function, to any precision. Thus, there exist non-isomorphic graphs that are indistinguishable by the folklore $k$-WL test but are distinguishable by $k$-\ourmethod. Additionally,  $k$-\ourmethod\ is able to compute the homomorphism count of $k$-apex forest graphs.
\end{theorem}

% \begin{theorem}[informal]
% \label{thm:ourmethod_vs_osan}
% Any  $k$-OSAN model  can be approximated to any precision by $k$-\ourmethod. 
% \end{theorem}

% As a consequence of Theorem~\ref{thm:ourmethod_vs_osan}, we obtain the following corollary, based on analogous results from \citet{qian2022ordered} and \citet{zhang2024beyond} concerning the expressive power of $k$-OSAN.
% \begin{corollary}
% \label{cor:ourmethod_vs_k_wl}
%     There exist non-isomorphic graphs that are indistinguishable by the folklore $k$-WL test but are distinguishable by $k$-\ourmethod. Additionally,  $k$-\ourmethod\ is able to compute the homomorphism count of $k$-apex forest graphs.
% \end{corollary}

% Corollary~\ref{cor:ourmethod_vs_k_wl} shows that a $k$-\ourmethod\ model—combining a base and downstream MPNN (both 1-WL bounded) with encoder networks (typically $(k)$- and $(k{+}1)$-IGN, bounded by folklore $k$-WL)—\textbf{strictly enhances} the expressivity of all components.

Theorem~\ref{thm:ourmethod_vs_osan} shows that a $k$-\ourmethod\ model, combining a base and downstream MPNN (both 1-WL bounded) with encoder networks (typically $(k)$ and $(k{+}1)$-IGN, bounded by folklore $k$-WL) \textbf{strictly enhances} the expressivity of all components.

The proof of Theorem  \ref{thm:ourmethod_vs_osan} relies on the analyticity of the activation functions used by our base MPNN and the use of higher-order derivatives. However, in what follows, we show that even when restricted to first-order derivatives and using the commonly employed ReLU activation, \ourmethod\ remains strictly more expressive than a widely used technique for enhancing GNN expressivity: incorporating Random Walk Structural Encodings (RWSEs)\citep{dwivedi2021graph} into a base MPNN.

\begin{theorem}[Informal]
\label{thm:ourmethod_vs_rwse}
Even when limited to first-order derivatives and ReLU activations, $1$-\ourmethod\ is strictly more expressive than MPNNs enhanced with random walk structural encodings.
\end{theorem}

\looseness=-1
The first part of Theorem~\ref{thm:ourmethod_vs_rwse} is constructive: it shows that a simple initialization of the base MPNN's weights yields derivatives equal to RWSEs. In our experiments, we use a slightly modified version of this initialization (see Appendix~\ref{app:experimental_details}), allowing \ourmethod to serve as a learnable extension of RWSE.

\textbf{Space and time complexity.} To conclude this section, we analyze the computational complexity of $k$-\ourmethod\ and compare it to other expressive architectures, namely, $(k{+}1)$-IGN and $k$-OSAN. 
We show that $k$-\ourmethod\ achieves better complexity when using relatively shallow base MPNNs, while maintaining comparable complexity with deeper ones. The primary source of computational overhead in $k$-\ourmethod\ lies in the derivative tensor computation, and the encoder network forward pass. We now analyze each of these components.

First, while $k$-\ourmethod\ computes derivative tensors $\tD_k^{(t)}$ with $O(n^{k+1})$ potential entries, these tensors are sparse for relatively shallow base MPNNs. Moreover, each $\tD_k^{(t)}$ can be efficiently computed from $\tD_k^{(t-1)}$. This is formalized in the following proposition:

\begin{proposition}
\label{prop:derivative_complexity}
In a $k$-\ourmethod\ model applied to a graph with $n$ nodes and maximum degree $d$ The number of non-zero entries in $\tD^{(t)}$ is at most $O\left(n \cdot \min \{ n^{k},\; d^{k \cdot t} \} \right)$. Additionally, each $\tD^{(t)}$ can be computed from $\tD^{(t-1)}$ in time $O(d\cdot n \cdot \min\{n^k, d^{k\cdot (t-1)}\})$.
\end{proposition}

\looseness=-1

Focusing next on the encoder networks, we show that they can be designed to exploit derivative sparsity for improved efficiency, while retaining the full expressivity of the $k$-OSAN architecture:

\begin{proposition}
\label{prop:encoder_complexity}
In a $k$-\ourmethod\ model, the encoder functions $\pool^{\text{node}}$ and $\pool^{\text{out}}$ can be chosen such that the model retains the expressive power of $k$-OSAN, while the computation of $\pool^{\text{node}}(\tD^{(T)})$ and $\pool^{\text{out}}(\tD^{\text{out}})$ has both time and space complexity $O\left(n \cdot \min\{n^k,\, d^{k \cdot T}\}\right)$.
\end{proposition}

\looseness=-1
Propositions~\ref{prop:derivative_complexity} and~\ref{prop:encoder_complexity} suggest that a $k$-\ourmethod\ model with a base MPNN of depth $T$ has space complexity $O\left(n \cdot \min\{n^k,\, d^{k \cdot T}\}\right)$ and time complexity $O\left(d \cdot n \cdot \min\{n^k,\, d^{k \cdot (T-1)}\}\right)$. In comparison, $k$-OSAN has space complexity $O(n^{k+1})$ and time complexity $O(d \cdot n^{k+1})$, while $(k{+}1)$-IGN incurs both time and space complexity of $O(n^{k+1})$. Assuming the input graph is sparse (i.e., $d \ll n$), $k$-\ourmethod\ is more efficient then $k$-OSAN  and $(k+1)$-IGN when the base MPNN is shallow ($d^{T} < n$), while all three models have comparable complexity when the base MPNN is deep ($d^{T} > n$).

\section{Experiments}
\label{sec:experiments}

\looseness=-1
Our experimental study is designed to validate the theoretical arguments from the previous section and to address the following guiding questions: \textbf{(Q1)} How does \ourmethod perform on real-world datasets when compared against strong GNN baselines? \textbf{(Q2)} Can \ourmethod scale to larger graphs that are beyond the reach of Subgraph GNNs, and how does it perform in this regime? \textbf{(Q3)}  How does the expressive power of \ourmethod compare with natural and widely used GNN baselines? We evaluate \ourmethod across seven benchmarks, with additional experimental details provided in Appendix~\ref{app:experimental_details}.

% Our experimental study is designed to validate the theoretical arguments from the previous section and to address the following guiding questions: \textbf{(Q1)} How does the expressive power of \ourmethod compare with natural and widely used GNN baselines? \textbf{(Q2)} Can \ourmethod scale to larger graphs that are beyond the reach of Subgraph GNNs, and how does it perform in this regime? \textbf{(Q3)} How does \ourmethod perform on real-world datasets when compared against strong GNN baselines?

\textbf{Baselines.} We compare \ourmethod against strong representatives from three natural families of GNNs. First, motivated by its connection to positional/structural encodings (PSEs, Section~\ref{sec:theory}), we consider \textbf{encoding-augmented MPNNs}, including Laplacian PEs~\citep{dwivedi2023benchmarking}, RWSEs~\citep{dwivedi2021graph}, SignNet~\citep{lim2022sign}, random node identifiers~\citep{abboud2020surprising, sato2021random}, as well as recent methods such as GPSE~\citep{canturk_gpse} and MOSE~\citep{bao2024homomorphism}. Second, since \ourmethod is theoretically related to \textbf{Subgraph GNNs}, we compare with representative models like GNN-AK~\citep{zhao2022from}, SUN~\citep{frasca2022understanding}, and Subgraphormer~\citep{bar-shalom2023subgraphormer}. Because such models often struggle to scale, we also include sampling-based variants such as Policy-Learn~\citep{bevilacqua2023efficient}, HyMN~\citep{southern2025balancing}, and Subgraphormer with random sampling. Finally, we benchmark against widely used and modern \textbf{general-purpose GNNs}, including GIN~\citep{xu2018powerful}, GCN~\citep{kipf2016semi}, GatedGCN~\citep{bresson2017residual}, GPS~\citep{rampavsek2022recipe}, and GraphViT~\citep{he2023generalization}. Across experiments we include representatives from each family, while the specific choice of baselines in each task reflects relevance and standard practice in prior work.

\begin{table}[t]
\setlength{\tabcolsep}{1pt}
\centering
\footnotesize
\vspace{-40pt}
\caption{Performance on \textsc{OGB} and \textsc{ZINC} datasets (4 seeds). 
\one{First} and \two{second} best scores are highlighted. Scores sharing a color are not statistically distinguishable based on Welch’s t-test with a relaxed threshold of $p < 0.2$.  “–” denotes results not previously reported, and “x” indicates that digits beyond this point were not provided.}

\begin{tabular}{l | c c c c}
\toprule
\multirow{2}{*}{{Method $\downarrow$ / Dataset $\rightarrow$}} & \textsc{ZINC-12K } & \textsc{moltox21} & \textsc{molbace} & \textsc{molhiv} \\
& \textsc{(MAE $\downarrow$)} & \textsc{(ROC-AUC $\uparrow$)} & \textsc{(ROC-AUC $\uparrow$)} & \textsc{(ROC-AUC $\uparrow$)}  \\ 
\midrule
\textbf{Common Baselines} \\
$\,$ GCN~\citep{kipf2016semi}        & 0.321$\pm$0.009 & 75.29$\pm$0.69 & 79.15$\pm$1.44 & 76.06$\pm$0.97 \\
$\,$  GIN~\citep{xu2018powerful}      & 0.163$\pm$0.004 & 74.91$\pm$0.51 & 72.97$\pm$4.00 & 75.58$\pm$1.40 \\
$\,$ PNA~\citep{corso2020principal}  & 0.761$\pm$0.002 & 73.30$\pm$1.1x & --             & 79.05$\pm$1.32 \\
$\,$  GPS~\citep{rampavsek2022recipe} & 0.070$\pm$0.004 & 75.70$\pm$0.40 & --             & 78.80$\pm$1.01 \\
$\,$ GraphViT \citep{he2023generalization}                        & 0.085$\pm$0.005 & \one{78.51$\pm$0.77} & --             & 77.92$\pm$1.49 \\
\midrule
\textbf{Subgraph GNNs} \\
$\,$ Reconstr. GNN~\citep{cotta2021reconstruction} & --             & 75.15$\pm$1.40 & --             & 76.32$\pm$1.40 \\
$\,$ GNN-AK+~\citep{zhao2022from}                  & 0.091$\pm$0.011 & --            & --             & \two{79.61$\pm$1.19} \\
$\,$ SUN (EGO+)~\citep{frasca2022understanding}    & 0.084$\pm$0.002 & --            & --             & \two{80.03$\pm$0.55} \\
$\,$ Full  \citep{bevilacqua2023efficient}                                        & 0.087$\pm$0.003 & 76.25$\pm$1.12 & 78.41$\pm$1.94 & 76.54$\pm$1.37 \\
$\,$ OSAN~\citep{qian2022ordered}                  & 0.177$\pm$0.016 & --            & 72.30$\pm$6.60 & -- \\
$\,$ Random \citep{bevilacqua2023efficient}                                        & 0.102$\pm$0.003 & 76.62$\pm$0.63 & 78.14$\pm$2.36 & 77.30$\pm$2.56 \\
$\,$ Policy-Learn \citep{bevilacqua2023efficient}                                  & 0.097$\pm$0.005 & \two{77.36$\pm$0.60} & 78.39$\pm$2.28 & 78.49$\pm$1.01 \\
% $\,$ Subgraphormer \citep{bar-shalom2024subgraphormer} & 0.067±0.007 & --  & \two{81.62$\pm$3.55}  & 80.38$\pm$1.92   \\
$\,$ Subgraphormer \citep{bar-shalom2024subgraphormer} & \one{0.063±0.001} & --  & \one{84.35$\pm$0.65}  & 79.58$\pm$0.35   \\
$\,$ HyMN \citep{southern2025balancing}                                          & 0.080$\pm$0.003 & \two{77.82$\pm$0.59} & \two{81.16$\pm$1.21} & \one{81.01$\pm$1.17} \\
\midrule
\textbf{PSEs} \\
$\,$ GIN + Laplacian PE~\citep{dwivedi2023benchmarking}  & 0.162$\pm$0.014 & 76.60$\pm$0.3x & 80.40$\pm$1.5x & 75.60$\pm$1.1x \\
$\,$ GIN + RWSE~\citep{dwivedi2021graph}                 & 0.128$\pm$0.005 & 76.30$\pm$0.5x & 79.60$\pm$2.8x & 78.10$\pm$1.5x \\
$\,$ SignNet~\citep{lim2022sign}                   & 0.102$\pm$0.002 & --             & --             & -- \\
$\,$ RNI~\citep{abboud2020surprising}              & 0.136±0.0070 & --             & 61.94$\pm$2.51 & 77.74±0.98 \\
$\,$ GSN~\citep{bouritsas2022improving} & 0.101$\pm$0.010 & 76.08$\pm$0.79 & 77.40$\pm$2.92 & \one{80.39$\pm$0.90} \\

$\,$ ENGNN~\citep{wang2025using}                   & 0.114$\pm$0.005 & --             & --             & 78.51$\pm$0.86 \\
$\,$ GPSE~\citep{canturk_gpse}                          & \two{0.065$\pm$0.003} & \two{77.40$\pm$0.8x} & \two{80.80$\pm$3.1x} & 78.15$\pm$1.33 \\
$\,$ MOSE \citep{bao2024homomorphism} & \one{0.062$\pm$0.002} & -- & -- & -- \\ 
\midrule
\textbf{Ours} \\
$\,$ \ourmethod                                    & \two{0.0666$\pm$0.0035} & \one{77.99$\pm$0.71} & \two{82.10$\pm$1.45} & \one{80.86$\pm$0.52} \\
\bottomrule
% \vspace{-10pt}
\end{tabular}
%}
\label{tab:ogb-full}
\end{table}

\looseness=-1
\textbf{OGB and ZINC.}  
To evaluate HOD-GNN’s real-world performance \textbf{(Q1)}, we benchmark it on standard graph property prediction datasets: ZINC~\citep{irwin2012zinc} for regression, and three molecular classification tasks from the OGB suite~\citep{hu2020open}—\texttt{molhiv}, \texttt{molbace}, and \texttt{moltox21}. These benchmarks provide standardized splits and are the de facto choice for assessing GNN performance. As shown in Table~\ref{tab:ogb-full}, \ourmethod delivers consistently excellent results across all tasks, standing out as the only model that consistently ranks within the top two tiers.

\textbf{Peptides.}  
Section~\ref{sec:theory} established that \ourmethod has improved computational complexity compared to Subgraph GNNs. To demonstrate its scalability in practice \textbf{(Q2)} and to further assess its performance on real-world data \textbf{(Q1)}, we evaluate \ourmethod on the \texttt{Peptides} datasets from the LRGB benchmark~\citep{dwivedi2022long}, where the goal is to predict global structural and functional properties of peptides represented as graphs. As stated in prior work~\citep{southern2025balancing, bar-shalom2023subgraphormer}, full-bag Subgraph GNNs cannot process these graphs directly using standard hardware, requiring the use of subsampling strategies that can reduce expressivity and introduce optimization challenges due to randomness.  In contrast, \ourmethod handles these graphs directly without subsampling.

\begin{wraptable}{r}{0.43\textwidth}
  \centering
  % \vspace{-10pt} % optional vertical \setlength{\tabcolsep}{3pt}
    \setlength{\tabcolsep}{3pt}

    \centering

    \caption{Performance on \textsc{Peptides} (4 seeds). \one{First} and \two{second} best scores are highlighted. Same color scores are not statistically distinguishable based on Welch’s t-test with a relaxed threshold of $p < 0.2$.}

    \label{tab:results_lrgb_complete}
    \scriptsize
    \begin{tabular}{@{}lcc@{}}
    \hline\toprule
    \multirow{2}{*}{\textbf{Model}} & \textbf{Peptides-func}  & \textbf{Peptides-struct}              
    \\
    % & \textbf{func} & \textbf{struct} %& \textbf{voc-sp}
    % \\
    & \scriptsize{AP $\uparrow$} & \scriptsize{MAE $\downarrow$} %& \scriptsize{F1 $\uparrow$}                             
    \\ \midrule  
    \textbf{Common Baselines} \\
    $\,$ GCN                        & 59.30${\pm0.23}$ & 0.3496${\pm0.0013}$ \\ %& 0.1268${\pm0.0060}$\\
    $\,$ GINE                       & 54.98${\pm0.79}$ & 0.3547${\pm0.0045}$ \\ %& 0.1265${\pm0.0076}$\\
    $\,$ GCNII                      & 55.43${\pm0.78}$ & 0.3471${\pm0.0010}$ \\ %& 0.1698${\pm0.0080}$\\
    $\,$ GatedGCN                   & 58.64${\pm0.77}$ & 0.3420${\pm0.0013}$ \\     $\,$ DIGL+MPNN+LapPE     & 68.30${\pm0.26}$         & 0.2616${\pm0.0018}$ \\ % & 0.2921${\pm0.0038}$\\
    $\,$ MixHop-GCN+LapPE    & 68.43${\pm0.49}$         & 0.2614${\pm0.0023}$ \\ % & 0.2218${\pm0.0174}$\\ 
    $\,$ DRew-GCN+LapPE             & \one{71.50${\pm0.44}$}   & 0.2536${\pm0.0015}$ \\ % & 0.1851${\pm0.0092}$\\    
    $\,$ SAN+LapPE         & 63.84${\pm1.21}$ & 0.2683${\pm0.0043}$           \\ % & \three{0.3230$_{\pm0.0039}$}\\
    $\,$ GraphGPS+LapPE    & 65.35${\pm0.41}$ & 0.2500${\pm0.0005}$   \\ % & 
     $\,$ Exphormer & 65.27${\pm0.43}$ & 0.2481${\pm0.0007}$ \\ 
     $\,$ GraphViT  & \two{69.19${\pm0.85}$} & 0.2474${\pm0.0016}$ \\
    \midrule
\textbf{Subgraph GNNs} \\
%$\,$ Reconstr. GNN~\citep{cotta2021reconstruction} &          -- & --     \\
%$\,$ GNN-AK+~\citep{zhao2022from}         &          -- & --            \\
%$\,$ SUN (EGO+)~\citep{frasca2022understanding}     &          -- & --   \\
%$\,$ Full  \citep{bevilacqua2023efficient} &          -- & --                                         \\
%$\,$ OSAN~\citep{qian2022ordered}  &          -- & --    \\
%$\,$ Random                          &          -- & --                 \\
$\,$ Policy-Learn  &        64.59${\pm0.18}$ & 	\two{0.2475${\pm0.0011}$ }                         \\
$\,$ Subgraphormer $30\%$& 64.15${\pm0.52}$  &  0.2494${\pm0.0020}$ \\
$\,$ HyMN  &  68.57${\pm0.55}$           &  \two{0.2464${\pm0.0013}$}\\
\midrule
\textbf{PSEs} \\
$\,$ GCN + Laplacian PE  & 62.18${\pm0.55}$ & 0.2492${\pm0.0019}$  \\
$\,$ GCN + RWSE  & 60.67${\pm0.69}$   & 0.2574${\pm0.0020 }$              \\
$\,$ SignNet &  63.14${\pm0.59}$ & --                   \\
%$\,$ RNI~\citep{abboud2020surprising}       &   --     & -- \\
%$\,$ ENGNN~\citep{wang2025using}              & --     & -- \\
$\,$ GPSE + GCN  &  63.16${\pm0.85}$      &   \two{0.2487${\pm0.0011}$ }                 \\
$\,$ GPSE  + GPS &  66.88${\pm1.51}$  & \one{0.2464${\pm0.0025}$ }  \\
$\,$ MOSE  &  63.5x${\pm1.1x}$  & 0.318${\pm0.010}$
 \\ 
 \midrule

    \textbf{Ours} \\
    $\,$  \ourmethod & \two{69.68${\pm0.56}$ } & \one{0.2450${\pm0.0011}$ } \\

    \bottomrule
    \end{tabular}
  \vspace{-10pt} % optional space below
\end{wraptable}

As shown in Table~\ref{tab:results_lrgb_complete}, \ourmethod surpasses all sampling-based Subgraph GNNs and is the only model that consistently ranks within the top two tiers, underscoring both its scalability and effectiveness on challenging real-world molecular tasks.

\looseness=-1
\textbf{Synthetic substructure counting.}  
Counting small substructures is a well-established proxy for evaluating the expressive power of GNNs~\citep{bouritsas2022improving, arvind2020weisfeiler}. In Section~\ref{sec:theory}, we prove that \ourmethod matches the expressivity of certain Subgraph GNNs and is strictly more powerful than MPNNs augmented with RWSEs. To empirically validate these results \textbf{(Q3)}, we evaluate the ability of \ourmethod to learn to count various local substructures, following the experimental protocol of \citet{huang2022boosting}. Table~\ref{tab:subgraph_counting} in Appendix \ref{app:additional_experiments} shows that \ourmethod achieves performance on par with Subgraph GNN baselines, while clearly surpassing MPNN+RWSE, confirming its strong practical expressivity.

\looseness=-1
\textbf{Ablations.} Appendix~\ref{app:additional_experiments} presents additional ablation studies. We first examine the effect of the activation function used by \ourmethod being analytic. In line with our theory, analytic activations consistently unlock higher expressive power, yielding the lowest MAE on substructure counting tasks. Non-analytic activations also enrich expressivity, substantially surpassing encoding augmented MPNNs. We then analyze the influence of the maximum derivative order used in \ourmethod on expressive power, observing steady improvements up to order~4, beyond which performance remains stable. Additional experiments compare runtime and memory against subgraph-based GNNs under different subgraph selection policies, showing that \ourmethod~achieves superior performance on both fronts. Finally, we evaluate the choice of backbone MPNN on real-world datasets, and observe consistently strong performance across GCN, GIN, and GPS.

 \looseness=-1
\textbf{Summary.}  
Across all experiments, we find consistent evidence supporting the guiding questions outlined above. \textbf{(A1)} Across the ZINC, OGB, and Peptides datasets, \ourmethod is the only architecture that consistently ranks within the top two tiers. \textbf{(A2)}, On the challenging \texttt{Peptides} datasets, \ourmethod scales to larger graphs that full-bag Subgraph GNNs cannot handle, while maintaining strong predictive performance. \textbf{(A3)}  On synthetic substructure counting, \ourmethod matches the expressivity of Subgraph GNNs while surpassing encoding-augmented MPNNs, showing that \ourmethod is \textbf{as expressive yet more scalable} than subgraph GNNs. Together, these findings establish \ourmethod as a scalable, expressive, and broadly effective GNN framework across synthetic, molecular, and large-scale real-world tasks.

\section{Conclusion}
\label{sec:conclusion}

\looseness=-1
We introduce \ourmethod, a GNN that enhances the expressivity of a base MPNN by leveraging its high-order derivatives. We provide a theoretical analysis of \ourmethod's expressive power, connecting it to the $k$-OSAN framework and RWSE encodings, and show that it can offer better scalability than comparable expressive models. Empirically, \ourmethod\ has strong performance across a range of benchmarks, matching or surpassing  encoding based methods, Subgraph GNNs, and other common baselines. Notably, \ourmethod\ scales to larger graphs than full-bag subgraph GNNs.

% and remains effective with deep base MPNNs and small hidden dimensions.

\looseness=-1
\textbf{Limitations and Future Work.}
First, while \ourmethod\ has favorable theoretical complexity (Section~\ref{sec:theory}), its practical efficiency relies on sparse matrix operations. A more efficient implementation of these operators has the potential to greatly enhance the scalability of \ourmethod. Secondly, the connection between MPNN derivatives and oversquashing or oversmoothing, alongside \ourmethod's strong performance with deep base MPNNs and small hidden dimensions, suggests a deeper link not fully addressed in this work. Exploring this connection is a promising direction for future research.

\section{Acknowledgments}
Y.E. is supported by the Zeff PhD Fellowship. Y.G. is supported by the the UKRI Engineering and Physical Sciences Research Council (EPSRC)
CDT in Autonomous and Intelligent Machines and Systems (grant reference EP/S024050/1). G.B. is supported by the Jacobs Qualcomm PhD Fellowship. F.F. conducted this work supported by an Aly Kaufman Post-Doctoral Fellowship. H.M. is the Robert J. Shillman Fellow, and is supported by the Israel Science Foundation through a personal grant (ISF 264/23) and an equipment grant (ISF 532/23).

\bibliographystyle{plainnat}   
\bibliography{references}

\begin{thebibliography}{81}
\providecommand{\natexlab}[1]{#1}
\providecommand{\url}[1]{\texttt{#1}}
\expandafter\ifx\csname urlstyle\endcsname\relax
  \providecommand{\doi}[1]{doi: #1}\else
  \providecommand{\doi}{doi: \begingroup \urlstyle{rm}\Url}\fi

\bibitem[Abboud et~al.(2020)Abboud, Ceylan, Grohe, and Lukasiewicz]{abboud2020surprising}
Ralph Abboud, Ismail~Ilkan Ceylan, Martin Grohe, and Thomas Lukasiewicz.
\newblock The surprising power of graph neural networks with random node initialization.
\newblock \emph{arXiv preprint arXiv:2010.01179}, 2020.

\bibitem[Akansha(2023)]{akansha2023over}
Singh Akansha.
\newblock Over-squashing in graph neural networks: A comprehensive survey.
\newblock \emph{arXiv preprint arXiv:2308.15568}, 2023.

\bibitem[Alon and Yahav(2020)]{alon2020bottleneck}
Uri Alon and Eran Yahav.
\newblock On the bottleneck of graph neural networks and its practical implications.
\newblock \emph{arXiv preprint arXiv:2006.05205}, 2020.

\bibitem[Arroyo et~al.(2025)Arroyo, Gravina, Gutteridge, Barbero, Gallicchio, Dong, Bronstein, and Vandergheynst]{arroyo2025vanishing}
{\'A}lvaro Arroyo, Alessio Gravina, Benjamin Gutteridge, Federico Barbero, Claudio Gallicchio, Xiaowen Dong, Michael Bronstein, and Pierre Vandergheynst.
\newblock On vanishing gradients, over-smoothing, and over-squashing in gnns: Bridging recurrent and graph learning.
\newblock \emph{arXiv preprint arXiv:2502.10818}, 2025.

\bibitem[Arvind et~al.(2020)Arvind, Fuhlbr{\"u}ck, K{\"o}bler, and Verbitsky]{arvind2020weisfeiler}
Vikraman Arvind, Frank Fuhlbr{\"u}ck, Johannes K{\"o}bler, and Oleg Verbitsky.
\newblock On weisfeiler-leman invariance: Subgraph counts and related graph properties.
\newblock \emph{Journal of Computer and System Sciences}, 113:\penalty0 42--59, 2020.

\bibitem[Baldassarre and Azizpour(2019)]{baldassarre2019explainability}
Federico Baldassarre and Hossein Azizpour.
\newblock Explainability techniques for graph convolutional networks.
\newblock \emph{arXiv preprint arXiv:1905.13686}, 2019.

\bibitem[Bao et~al.(2024)Bao, Jin, Bronstein, Ceylan, and Lanzinger]{bao2024homomorphism}
Linus Bao, Emily Jin, Michael Bronstein, {\.I}smail~{\.I}lkan Ceylan, and Matthias Lanzinger.
\newblock Homomorphism counts as structural encodings for graph learning.
\newblock \emph{arXiv preprint arXiv:2410.18676}, 2024.

\bibitem[Bar-Shalom et~al.(2023)Bar-Shalom, Bevilacqua, and Maron]{bar-shalom2023subgraphormer}
Guy Bar-Shalom, Beatrice Bevilacqua, and Haggai Maron.
\newblock Subgraphormer: Subgraph {GNN}s meet graph transformers.
\newblock In \emph{NeurIPS 2023 Workshop: New Frontiers in Graph Learning}, 2023.
\newblock URL \url{https://openreview.net/forum?id=e8ba9Hu1mM}.

\bibitem[Bar-Shalom et~al.(2024{\natexlab{a}})Bar-Shalom, Bevilacqua, and Maron]{bar-shalom2024subgraphormer}
Guy Bar-Shalom, Beatrice Bevilacqua, and Haggai Maron.
\newblock Subgraphormer: Unifying subgraph {GNN}s and graph transformers via graph products.
\newblock In \emph{Forty-first International Conference on Machine Learning}, 2024{\natexlab{a}}.
\newblock URL \url{https://openreview.net/forum?id=6djDWVTUEq}.

\bibitem[Bar-Shalom et~al.(2024{\natexlab{b}})Bar-Shalom, Eitan, Frasca, and Maron]{bar2024flexible}
Guy Bar-Shalom, Yam Eitan, Fabrizio Frasca, and Haggai Maron.
\newblock A flexible, equivariant framework for subgraph gnns via graph products and graph coarsening.
\newblock \emph{arXiv preprint arXiv:2406.09291}, 2024{\natexlab{b}}.

\bibitem[Bevilacqua et~al.(2021)Bevilacqua, Frasca, Lim, Srinivasan, Cai, Balamurugan, Bronstein, and Maron]{bevilacqua2021equivariant}
Beatrice Bevilacqua, Fabrizio Frasca, Derek Lim, Balasubramaniam Srinivasan, Chen Cai, Gopinath Balamurugan, Michael~M Bronstein, and Haggai Maron.
\newblock Equivariant subgraph aggregation networks.
\newblock \emph{arXiv preprint arXiv:2110.02910}, 2021.

\bibitem[Bevilacqua et~al.(2024)Bevilacqua, Eliasof, Meirom, Ribeiro, and Maron]{bevilacqua2023efficient}
Beatrice Bevilacqua, Moshe Eliasof, Eli Meirom, Bruno Ribeiro, and Haggai Maron.
\newblock Efficient subgraph gnns by learning effective selection policies.
\newblock \emph{International Conference on Learning Representations}, 2024.

\bibitem[Bodnar et~al.(2021)Bodnar, Frasca, Otter, Wang, Lio, Montufar, and Bronstein]{bodnar2021weisfeiler}
Cristian Bodnar, Fabrizio Frasca, Nina Otter, Yuguang Wang, Pietro Lio, Guido~F Montufar, and Michael Bronstein.
\newblock Weisfeiler and lehman go cellular: Cw networks.
\newblock \emph{Advances in Neural Information Processing Systems}, 34:\penalty0 2625--2640, 2021.

\bibitem[Bouritsas et~al.(2022)Bouritsas, Frasca, Zafeiriou, and Bronstein]{bouritsas2022improving}
Giorgos Bouritsas, Fabrizio Frasca, Stefanos Zafeiriou, and Michael~M Bronstein.
\newblock Improving graph neural network expressivity via subgraph isomorphism counting.
\newblock \emph{IEEE Transactions on Pattern Analysis and Machine Intelligence}, 45\penalty0 (1):\penalty0 657--668, 2022.

\bibitem[Bresson and Laurent(2017)]{bresson2017residual}
Xavier Bresson and Thomas Laurent.
\newblock Residual gated graph convnets.
\newblock \emph{arXiv preprint arXiv:1711.07553}, 2017.

\bibitem[Cant{\"u}rk et~al.(2024)Cant{\"u}rk, Liu, Lapointe-Gagn{\'e}, L{\'e}tourneau, Wolf, Beaini, and Ramp{\'a}{\v{s}}ek]{canturk_gpse}
Semih Cant{\"u}rk, Renming Liu, Olivier Lapointe-Gagn{\'e}, Vincent L{\'e}tourneau, Guy Wolf, Dominique Beaini, and Ladislav Ramp{\'a}{\v{s}}ek.
\newblock Graph positional and structural encoder.
\newblock In \emph{Forty-first International Conference on Machine Learning}, 2024.
\newblock URL \url{https://openreview.net/forum?id=UTSCK582Yo}.

\bibitem[Cappart et~al.(2023)Cappart, Ch{\'e}telat, Khalil, Lodi, Morris, and Veli{\v{c}}kovi{\'c}]{cappart2023combinatorial}
Quentin Cappart, Didier Ch{\'e}telat, Elias~B Khalil, Andrea Lodi, Christopher Morris, and Petar Veli{\v{c}}kovi{\'c}.
\newblock Combinatorial optimization and reasoning with graph neural networks.
\newblock \emph{Journal of Machine Learning Research}, 24\penalty0 (130):\penalty0 1--61, 2023.

\bibitem[Chen et~al.(2020)Chen, Chen, Villar, and Bruna]{chen2020can}
Zhengdao Chen, Lei Chen, Soledad Villar, and Joan Bruna.
\newblock Can graph neural networks count substructures?
\newblock \emph{Advances in neural information processing systems}, 33:\penalty0 10383--10395, 2020.

\bibitem[Corso et~al.(2020)Corso, Cavalleri, Beaini, Li{\`o}, and Veli{\v{c}}kovi{\'c}]{corso2020principal}
Gabriele Corso, Luca Cavalleri, Dominique Beaini, Pietro Li{\`o}, and Petar Veli{\v{c}}kovi{\'c}.
\newblock Principal neighbourhood aggregation for graph nets.
\newblock \emph{Advances in Neural Information Processing Systems}, 33:\penalty0 13260--13271, 2020.

\bibitem[Cotta et~al.(2021)Cotta, Morris, and Ribeiro]{cotta2021reconstruction}
Leonardo Cotta, Christopher Morris, and Bruno Ribeiro.
\newblock Reconstruction for powerful graph representations.
\newblock In \emph{Advances in Neural Information Processing Systems}, volume~34, 2021.

\bibitem[Di~Giovanni et~al.(2023{\natexlab{a}})Di~Giovanni, Giusti, Barbero, Luise, Lio, and Bronstein]{di2023over}
Francesco Di~Giovanni, Lorenzo Giusti, Federico Barbero, Giulia Luise, Pietro Lio, and Michael~M Bronstein.
\newblock On over-squashing in message passing neural networks: The impact of width, depth, and topology.
\newblock In \emph{International Conference on Machine Learning}, pages 7865--7885. PMLR, 2023{\natexlab{a}}.

\bibitem[Di~Giovanni et~al.(2023{\natexlab{b}})Di~Giovanni, Rusch, Bronstein, Deac, Lackenby, Mishra, and Veli{\v{c}}kovi{\'c}]{didoes}
Francesco Di~Giovanni, T~Konstantin Rusch, Michael Bronstein, Andreea Deac, Marc Lackenby, Siddhartha Mishra, and Petar Veli{\v{c}}kovi{\'c}.
\newblock How does over-squashing affect the power of gnns?
\newblock \emph{Transactions on Machine Learning Research}, 2023{\natexlab{b}}.

\bibitem[Dwivedi et~al.(2021)Dwivedi, Luu, Laurent, Bengio, and Bresson]{dwivedi2021graph}
Vijay~Prakash Dwivedi, Anh~Tuan Luu, Thomas Laurent, Yoshua Bengio, and Xavier Bresson.
\newblock Graph neural networks with learnable structural and positional representations.
\newblock \emph{arXiv preprint arXiv:2110.07875}, 2021.

\bibitem[Dwivedi et~al.(2022)Dwivedi, Ramp{\'a}{\v{s}}ek, Galkin, Parviz, Wolf, Luu, and Beaini]{dwivedi2022long}
Vijay~Prakash Dwivedi, Ladislav Ramp{\'a}{\v{s}}ek, Michael Galkin, Ali Parviz, Guy Wolf, Anh~Tuan Luu, and Dominique Beaini.
\newblock Long range graph benchmark.
\newblock \emph{Advances in Neural Information Processing Systems}, 35:\penalty0 22326--22340, 2022.

\bibitem[Dwivedi et~al.(2023)Dwivedi, Joshi, Luu, Laurent, Bengio, and Bresson]{dwivedi2023benchmarking}
Vijay~Prakash Dwivedi, Chaitanya~K Joshi, Anh~Tuan Luu, Thomas Laurent, Yoshua Bengio, and Xavier Bresson.
\newblock Benchmarking graph neural networks.
\newblock \emph{Journal of Machine Learning Research}, 24\penalty0 (43):\penalty0 1--48, 2023.

\bibitem[Eitan et~al.(2024)Eitan, Gelberg, Bar-Shalom, Frasca, Bronstein, and Maron]{eitan2024topological}
Yam Eitan, Yoav Gelberg, Guy Bar-Shalom, Fabrizio Frasca, Michael Bronstein, and Haggai Maron.
\newblock Topological blind spots: Understanding and extending topological deep learning through the lens of expressivity.
\newblock \emph{arXiv preprint arXiv:2408.05486}, 2024.

\bibitem[Eliasof et~al.(2023)Eliasof, Frasca, Bevilacqua, Treister, Chechik, and Maron]{eliasof2023graph}
Moshe Eliasof, Fabrizio Frasca, Beatrice Bevilacqua, Eran Treister, Gal Chechik, and Haggai Maron.
\newblock Graph positional encoding via random feature propagation.
\newblock In \emph{International Conference on Machine Learning}, pages 9202--9223. PMLR, 2023.

\bibitem[Fey and Lenssen(2019)]{fey2019fast}
Matthias Fey and Jan~Eric Lenssen.
\newblock Fast graph representation learning with pytorch geometric.
\newblock \emph{arXiv preprint arXiv:1903.02428}, 2019.

\bibitem[Frasca et~al.(2022)Frasca, Bevilacqua, Bronstein, and Maron]{frasca2022understanding}
Fabrizio Frasca, Beatrice Bevilacqua, Michael Bronstein, and Haggai Maron.
\newblock Understanding and extending subgraph gnns by rethinking their symmetries.
\newblock \emph{Advances in Neural Information Processing Systems}, 35:\penalty0 31376--31390, 2022.

\bibitem[Gelberg et~al.(2025)Gelberg, Eitan, Navon, Shamsian, Putterman, and Maron]{gelberggradmetanet}
Yoav Gelberg, Yam Eitan, Aviv Navon, Aviv Shamsian, Theo Putterman, and Haggai Maron.
\newblock Gradmetanet: An equivariant architecture for learning on gradients.
\newblock In \emph{Workshop on Neural Network Weights as a New Data Modality}, 2025.

\bibitem[Gilmer et~al.(2017)Gilmer, Schoenholz, Riley, Vinyals, and Dahl]{gilmer2017neural}
Justin Gilmer, Samuel~S Schoenholz, Patrick~F Riley, Oriol Vinyals, and George~E Dahl.
\newblock Neural message passing for quantum chemistry.
\newblock In \emph{International conference on machine learning}, pages 1263--1272. PMLR, 2017.

\bibitem[Hardy(2006)]{hardy2006combinatorics}
Michael Hardy.
\newblock Combinatorics of partial derivatives.
\newblock \emph{arXiv preprint math/0601149}, 2006.

\bibitem[He et~al.(2023)He, Hooi, Laurent, Perold, LeCun, and Bresson]{he2023generalization}
Xiaoxin He, Bryan Hooi, Thomas Laurent, Adam Perold, Yann LeCun, and Xavier Bresson.
\newblock A generalization of vit/mlp-mixer to graphs.
\newblock In \emph{International conference on machine learning}, pages 12724--12745. PMLR, 2023.

\bibitem[Higham(2002)]{higham2002accuracy}
Nicholas~J Higham.
\newblock \emph{Accuracy and stability of numerical algorithms}.
\newblock SIAM, 2002.

\bibitem[Hu et~al.(2020{\natexlab{a}})Hu, Xiong, Qu, Yuan, C{\^o}t{\'e}, Liu, and Tang]{hu2020graph}
Shengding Hu, Zheng Xiong, Meng Qu, Xingdi Yuan, Marc-Alexandre C{\^o}t{\'e}, Zhiyuan Liu, and Jian Tang.
\newblock Graph policy network for transferable active learning on graphs.
\newblock \emph{Advances in Neural Information Processing Systems}, 33:\penalty0 10174--10185, 2020{\natexlab{a}}.

\bibitem[Hu et~al.(2020{\natexlab{b}})Hu, Fey, Zitnik, Dong, Ren, Liu, Catasta, and Leskovec]{hu2020open}
Weihua Hu, Matthias Fey, Marinka Zitnik, Yuxiao Dong, Hongyu Ren, Bowen Liu, Michele Catasta, and Jure Leskovec.
\newblock Open graph benchmark: Datasets for machine learning on graphs.
\newblock \emph{Advances in neural information processing systems}, 33:\penalty0 22118--22133, 2020{\natexlab{b}}.

\bibitem[Huang et~al.(2022)Huang, Peng, Ma, and Zhang]{huang2022boosting}
Yinan Huang, Xingang Peng, Jianzhu Ma, and Muhan Zhang.
\newblock Boosting the cycle counting power of graph neural networks with i$^{2}$-gnns.
\newblock In \emph{The Eleventh International Conference on Learning Representations}, 2022.

\bibitem[Irwin et~al.(2012)Irwin, Sterling, Mysinger, Bolstad, and Coleman]{irwin2012zinc}
John~J Irwin, Teague Sterling, Michael~M Mysinger, Eric~S Bolstad, and Ryan~G Coleman.
\newblock Zinc--a free database of commercially available compounds for virtual screening.
\newblock \emph{Journal of Chemical Information and Modeling}, 52\penalty0 (7):\penalty0 1757--1768, 2012.

\bibitem[Jegelka(2022)]{jegelka2022theory}
Stefanie Jegelka.
\newblock Theory of graph neural networks: Representation and learning.
\newblock In \emph{The International Congress of Mathematicians}, pages 1--23, 2022.

\bibitem[Keren~Taraday et~al.(2024)Keren~Taraday, David, and Baskin]{keren2024sequential}
Mitchell Keren~Taraday, Almog David, and Chaim Baskin.
\newblock Sequential signal mixing aggregation for message passing graph neural networks.
\newblock \emph{Advances in Neural Information Processing Systems}, 37:\penalty0 93985--94021, 2024.

\bibitem[Kipf and Welling(2016)]{kipf2016semi}
Thomas~N Kipf and Max Welling.
\newblock Semi-supervised classification with graph convolutional networks.
\newblock \emph{International Conference on Learning Representations}, 2016.

\bibitem[Li and Leskovec(2022)]{li2022expressive}
Pan Li and Jure Leskovec.
\newblock The expressive power of graph neural networks.
\newblock \emph{Graph Neural Networks: Foundations, Frontiers, and Applications}, pages 63--98, 2022.

\bibitem[Lim et~al.(2022)Lim, Robinson, Zhao, Smidt, Sra, Maron, and Jegelka]{lim2022sign}
Derek Lim, Joshua Robinson, Lingxiao Zhao, Tess Smidt, Suvrit Sra, Haggai Maron, and Stefanie Jegelka.
\newblock Sign and basis invariant networks for spectral graph representation learning.
\newblock \emph{arXiv preprint arXiv:2202.13013}, 2022.

\bibitem[Loshchilov and Hutter(2017)]{loshchilov2017decoupled}
Ilya Loshchilov and Frank Hutter.
\newblock Decoupled weight decay regularization.
\newblock \emph{arXiv preprint arXiv:1711.05101}, 2017.

\bibitem[Luo et~al.(2020)Luo, Cheng, Xu, Yu, Zong, Chen, and Zhang]{luo2020parameterized}
Dongsheng Luo, Wei Cheng, Dongkuan Xu, Wenchao Yu, Bo~Zong, Haifeng Chen, and Xiang Zhang.
\newblock Parameterized explainer for graph neural network.
\newblock \emph{Advances in neural information processing systems}, 33:\penalty0 19620--19631, 2020.

\bibitem[Maron et~al.(2018)Maron, Ben-Hamu, Shamir, and Lipman]{maron2018invariant}
Haggai Maron, Heli Ben-Hamu, Nadav Shamir, and Yaron Lipman.
\newblock Invariant and equivariant graph networks.
\newblock \emph{arXiv preprint arXiv:1812.09902}, 2018.

\bibitem[Maron et~al.(2019)Maron, Ben-Hamu, Serviansky, and Lipman]{maron2019provably}
Haggai Maron, Heli Ben-Hamu, Hadar Serviansky, and Yaron Lipman.
\newblock Provably powerful graph networks.
\newblock \emph{Advances in neural information processing systems}, 32, 2019.

\bibitem[Mitchell et~al.(2021)Mitchell, Lin, Bosselut, Finn, and Manning]{mitchell2021fast}
Eric Mitchell, Charles Lin, Antoine Bosselut, Chelsea Finn, and Christopher~D Manning.
\newblock Fast model editing at scale.
\newblock \emph{arXiv preprint arXiv:2110.11309}, 2021.

\bibitem[Monti et~al.(2019)Monti, Frasca, Eynard, Mannion, and Bronstein]{monti2019fake}
Federico Monti, Fabrizio Frasca, Davide Eynard, Damon Mannion, and Michael~M Bronstein.
\newblock Fake news detection on social media using geometric deep learning.
\newblock \emph{arXiv preprint arXiv:1902.06673}, 2019.

\bibitem[Morris et~al.(2019)Morris, Ritzert, Fey, Hamilton, Lenssen, Rattan, and Grohe]{morris2019weisfeiler}
Christopher Morris, Martin Ritzert, Matthias Fey, William~L Hamilton, Jan~Eric Lenssen, Gaurav Rattan, and Martin Grohe.
\newblock Weisfeiler and leman go neural: Higher-order graph neural networks.
\newblock In \emph{Proceedings of the AAAI conference on artificial intelligence}, volume~33, pages 4602--4609, 2019.

\bibitem[Morris et~al.(2023)Morris, Lipman, Maron, Rieck, Kriege, Grohe, Fey, and Borgwardt]{morris2023weisfeiler}
Christopher Morris, Yaron Lipman, Haggai Maron, Bastian Rieck, Nils~M Kriege, Martin Grohe, Matthias Fey, and Karsten Borgwardt.
\newblock Weisfeiler and leman go machine learning: The story so far.
\newblock \emph{The Journal of Machine Learning Research}, 24\penalty0 (1):\penalty0 15865--15923, 2023.

\bibitem[Papp and Wattenhofer(2022)]{papp2022theoretical}
P{\'a}l~Andr{\'a}s Papp and Roger Wattenhofer.
\newblock A theoretical comparison of graph neural network extensions.
\newblock In \emph{International Conference on Machine Learning}, pages 17323--17345. PMLR, 2022.

\bibitem[Paszke et~al.(2019)Paszke, Gross, Massa, Lerer, Bradbury, Chanan, Killeen, Lin, Gimelshein, Antiga, et~al.]{paszke2019pytorch}
Adam Paszke, Sam Gross, Francisco Massa, Adam Lerer, James Bradbury, Gregory Chanan, Trevor Killeen, Zeming Lin, Natalia Gimelshein, Luca Antiga, et~al.
\newblock Pytorch: An imperative style, high-performance deep learning library.
\newblock \emph{Advances in neural information processing systems}, 32, 2019.

\bibitem[Pellizzoni et~al.(2024)Pellizzoni, Schulz, Chen, and Borgwardt]{pellizzoni2024expressivity}
Paolo Pellizzoni, Till~Hendrik Schulz, Dexiong Chen, and Karsten Borgwardt.
\newblock On the expressivity and sample complexity of node-individualized graph neural networks.
\newblock \emph{Advances in Neural Information Processing Systems}, 37:\penalty0 120221--120251, 2024.

\bibitem[Pope et~al.(2019)Pope, Kolouri, Rostami, Martin, and Hoffmann]{pope2019explainability}
Phillip~E Pope, Soheil Kolouri, Mohammad Rostami, Charles~E Martin, and Heiko Hoffmann.
\newblock Explainability methods for graph convolutional neural networks.
\newblock In \emph{Proceedings of the IEEE/CVF conference on computer vision and pattern recognition}, pages 10772--10781, 2019.

\bibitem[Puny et~al.(2023)Puny, Lim, Kiani, Maron, and Lipman]{puny2023equivariant}
Omri Puny, Derek Lim, Bobak Kiani, Haggai Maron, and Yaron Lipman.
\newblock Equivariant polynomials for graph neural networks.
\newblock In \emph{International Conference on Machine Learning}, pages 28191--28222. PMLR, 2023.

\bibitem[Qian et~al.(2022)Qian, Rattan, Geerts, Niepert, and Morris]{qian2022ordered}
Chendi Qian, Gaurav Rattan, Floris Geerts, Mathias Niepert, and Christopher Morris.
\newblock Ordered subgraph aggregation networks.
\newblock \emph{Advances in Neural Information Processing Systems}, 35:\penalty0 21030--21045, 2022.

\bibitem[Ramp{\'a}{\v{s}}ek et~al.(2022)Ramp{\'a}{\v{s}}ek, Galkin, Dwivedi, Luu, Wolf, and Beaini]{rampavsek2022recipe}
Ladislav Ramp{\'a}{\v{s}}ek, Michael Galkin, Vijay~Prakash Dwivedi, Anh~Tuan Luu, Guy Wolf, and Dominique Beaini.
\newblock Recipe for a general, powerful, scalable graph transformer.
\newblock \emph{Advances in Neural Information Processing Systems}, 35:\penalty0 14501--14515, 2022.

\bibitem[Read and Wilson(1998)]{read1998atlas}
Ronald~C Read and Robin~J Wilson.
\newblock \emph{An atlas of graphs}.
\newblock Oxford University Press, 1998.

\bibitem[Rieck et~al.(2019)Rieck, Bock, and Borgwardt]{rieck2019persistent}
Bastian Rieck, Christian Bock, and Karsten Borgwardt.
\newblock A persistent weisfeiler-lehman procedure for graph classification.
\newblock In \emph{International Conference on Machine Learning}, pages 5448--5458. PMLR, 2019.

\bibitem[Sato(2020)]{sato2020survey}
Ryoma Sato.
\newblock A survey on the expressive power of graph neural networks.
\newblock \emph{arXiv preprint arXiv:2003.04078}, 2020.

\bibitem[Sato et~al.(2021)Sato, Yamada, and Kashima]{sato2021random}
Ryoma Sato, Makoto Yamada, and Hisashi Kashima.
\newblock Random features strengthen graph neural networks.
\newblock In \emph{Proceedings of the 2021 SIAM international conference on data mining (SDM)}, pages 333--341. SIAM, 2021.

\bibitem[Segol and Lipman(2019)]{segol2019universal}
Nimrod Segol and Yaron Lipman.
\newblock On universal equivariant set networks.
\newblock \emph{arXiv preprint arXiv:1910.02421}, 2019.

\bibitem[Southern et~al.(2025)Southern, Eitan, Bar-Shalom, Bronstein, Maron, and Frasca]{southern2025balancing}
Joshua Southern, Yam Eitan, Guy Bar-Shalom, Michael~M. Bronstein, Haggai Maron, and Fabrizio Frasca.
\newblock Balancing efficiency and expressiveness: Subgraph {GNN}s with walk-based centrality, 2025.
\newblock URL \url{https://openreview.net/forum?id=2hbgKYuao1}.

\bibitem[T{\"o}nshoff et~al.(2023)T{\"o}nshoff, Ritzert, Rosenbluth, and Grohe]{tonshoff2023did}
Jan T{\"o}nshoff, Martin Ritzert, Eran Rosenbluth, and Martin Grohe.
\newblock Where did the gap go? reassessing the long-range graph benchmark.
\newblock \emph{arXiv preprint arXiv:2309.00367}, 2023.

\bibitem[Topping et~al.(2021)Topping, Di~Giovanni, Chamberlain, Dong, and Bronstein]{topping2021understanding}
Jake Topping, Francesco Di~Giovanni, Benjamin~Paul Chamberlain, Xiaowen Dong, and Michael~M Bronstein.
\newblock Understanding over-squashing and bottlenecks on graphs via curvature.
\newblock \emph{arXiv preprint arXiv:2111.14522}, 2021.

\bibitem[Wang and Zhang(2025)]{wang2025using}
Xiyuan Wang and Muhan Zhang.
\newblock Using random noise equivariantly to boost graph neural networks universally.
\newblock \emph{arXiv preprint arXiv:2502.02479}, 2025.

\bibitem[Weisfeiler and Leman(1968)]{weisfeiler1968reduction}
Boris Weisfeiler and Andrei Leman.
\newblock The reduction of a graph to canonical form and the algebra which appears therein.
\newblock \emph{nti, Series}, 2\penalty0 (9):\penalty0 12--16, 1968.

\bibitem[Wong et~al.(2024)Wong, Zheng, Valeri, Donghia, Anahtar, Omori, Li, Cubillos-Ruiz, Krishnan, Jin, et~al.]{wong2024discovery}
Felix Wong, Erica~J Zheng, Jacqueline~A Valeri, Nina~M Donghia, Melis~N Anahtar, Satotaka Omori, Alicia Li, Andres Cubillos-Ruiz, Aarti Krishnan, Wengong Jin, et~al.
\newblock Discovery of a structural class of antibiotics with explainable deep learning.
\newblock \emph{Nature}, 626\penalty0 (7997):\penalty0 177--185, 2024.

\bibitem[Xu et~al.(2022)Xu, Wang, Jiang, Fan, and Wang]{xu2022signal}
Dejia Xu, Peihao Wang, Yifan Jiang, Zhiwen Fan, and Zhangyang Wang.
\newblock Signal processing for implicit neural representations.
\newblock \emph{Advances in Neural Information Processing Systems}, 35:\penalty0 13404--13418, 2022.

\bibitem[Xu et~al.(2018)Xu, Hu, Leskovec, and Jegelka]{xu2018powerful}
Keyulu Xu, Weihua Hu, Jure Leskovec, and Stefanie Jegelka.
\newblock How powerful are graph neural networks?
\newblock \emph{International Conference on Learning Representations}, 2018.

\bibitem[Yan et~al.(2024)Yan, Zhou, Gao, Tang, and Zhang]{yan2024efficient}
Zuoyu Yan, Junru Zhou, Liangcai Gao, Zhi Tang, and Muhan Zhang.
\newblock An efficient subgraph gnn with provable substructure counting power.
\newblock In \emph{Proceedings of the 30th ACM SIGKDD Conference on Knowledge Discovery and Data Mining}, pages 3702--3713, 2024.

\bibitem[Ying et~al.(2019)Ying, Bourgeois, You, Zitnik, and Leskovec]{ying2019gnnexplainer}
Zhitao Ying, Dylan Bourgeois, Jiaxuan You, Marinka Zitnik, and Jure Leskovec.
\newblock Gnnexplainer: Generating explanations for graph neural networks.
\newblock \emph{Advances in neural information processing systems}, 32, 2019.

\bibitem[You et~al.(2021)You, Gomes-Selman, Ying, and Leskovec]{you2021identity}
Jiaxuan You, Jonathan~M Gomes-Selman, Rex Ying, and Jure Leskovec.
\newblock Identity-aware graph neural networks.
\newblock In \emph{Proceedings of the AAAI conference on artificial intelligence}, volume~35, pages 10737--10745, 2021.

\bibitem[Zaheer et~al.(2017)Zaheer, Kottur, Ravanbakhsh, Poczos, Salakhutdinov, and Smola]{zaheer2017deep}
Manzil Zaheer, Satwik Kottur, Siamak Ravanbakhsh, Barnabas Poczos, Russ~R Salakhutdinov, and Alexander~J Smola.
\newblock Deep sets.
\newblock \emph{Advances in neural information processing systems}, 30, 2017.

\bibitem[Zhang et~al.(2024{\natexlab{a}})Zhang, Fan, Liu, Huang, Zhao, Huang, and Liu]{zhang2024expressive-b}
Bingxu Zhang, Changjun Fan, Shixuan Liu, Kuihua Huang, Xiang Zhao, Jincai Huang, and Zhong Liu.
\newblock The expressive power of graph neural networks: A survey.
\newblock \emph{IEEE Transactions on Knowledge and Data Engineering}, 2024{\natexlab{a}}.

\bibitem[Zhang et~al.(2023{\natexlab{a}})Zhang, Feng, Du, He, and Wang]{zhang2023complete}
Bohang Zhang, Guhao Feng, Yiheng Du, Di~He, and Liwei Wang.
\newblock A complete expressiveness hierarchy for subgraph gnns via subgraph weisfeiler-lehman tests.
\newblock In \emph{International Conference on Machine Learning}, pages 41019--41077. PMLR, 2023{\natexlab{a}}.

\bibitem[Zhang et~al.(2023{\natexlab{b}})Zhang, Luo, Wang, and He]{zhang2023rethinking}
Bohang Zhang, Shengjie Luo, Liwei Wang, and Di~He.
\newblock Rethinking the expressive power of gnns via graph biconnectivity.
\newblock \emph{arXiv preprint arXiv:2301.09505}, 2023{\natexlab{b}}.

\bibitem[Zhang et~al.(2024{\natexlab{b}})Zhang, Gai, Du, Ye, He, and Wang]{zhang2024beyond}
Bohang Zhang, Jingchu Gai, Yiheng Du, Qiwei Ye, Di~He, and Liwei Wang.
\newblock Beyond weisfeiler-lehman: A quantitative framework for gnn expressiveness.
\newblock \emph{arXiv preprint arXiv:2401.08514}, 2024{\natexlab{b}}.

\bibitem[Zhang and Li(2021)]{zhang2021nested}
Muhan Zhang and Pan Li.
\newblock Nested graph neural networks.
\newblock In \emph{Advances in Neural Information Processing Systems}, volume~34, 2021.

\bibitem[Zhao et~al.(2022)Zhao, Jin, Akoglu, and Shah]{zhao2022from}
Lingxiao Zhao, Wei Jin, Leman Akoglu, and Neil Shah.
\newblock From stars to subgraphs: Uplifting any {GNN} with local structure awareness.
\newblock In \emph{International Conference on Learning Representations}, 2022.

\end{thebibliography}

\appendix
\newpage

\section{Previous Work}
\label{apx:prev_work}
\textbf{Expressive Power and Hierarchies in GNNs.}  
The expressive power of GNNs is often measured by their ability to distinguish non-isomorphic graphs. Foundational results~\citep{morris2019weisfeiler, xu2018powerful} show that standard MPNNs are bounded by the 1-Weisfeiler-Lehman (1-WL) test~\citep{weisfeiler1968reduction}, motivating the development of more expressive architectures, see \citep{sato2020survey,morris2023weisfeiler, jegelka2022theory, li2022expressive,zhang2024expressive-b}
 for comprehensive surveys. For instance, \citet{morris2019weisfeiler} and \citet{maron2018invariant} introduced GNN hierarchies matching the expressivity of the $k$-WL test at a computational cost of $\mathcal{O}(n^k)$ in both time and memory. Other approaches include equivariant polynomial-based models~\citep{maron2019provably, puny2023equivariant}, Subgraph GNNs~\citep{ zhang2021nested, cotta2021reconstruction, bevilacqua2021equivariant, frasca2022understanding,zhang2023rethinking,zhang2023complete,bar2024flexible} topologically enhanced GNNs \citep{rieck2019persistent,bodnar2021weisfeiler, eitan2024topological} and more. A complementary line of work improves expressivity by enriching node features with informative structural descriptors, such as substructure and homomorphism counts \citep{bouritsas2022improving, bao2024homomorphism}, random node features \citep{abboud2020surprising, sato2021random, eliasof2023graph}, or spectral methods \citep{dwivedi2023benchmarking, lim2022sign}.

\textbf{Derivatives of GNNs.} 
Derivatives frequently appear in the analysis of GNNs. A prominent example is the study of \textit{oversquashing}-the failure of information to propagate through graph structures~\citep{alon2020bottleneck, topping2021understanding, di2023over, didoes}. A central tool in analyzing oversquashing is the use of derivatives: specifically, the gradients of final node representations with respect to initial features~\citep[see][]{di2023over}, and mixed output derivatives with respect to pairs of input nodes~\citep[see][]{didoes}. For a comprehensive overview of oversquashing, see ~\citet{akansha2023over}. Node derivatives also play a key role in GNN explainability~\citep{ying2019gnnexplainer,luo2020parameterized,baldassarre2019explainability,pope2019explainability}. Gradient-based approaches such as Sensitivity Analysis, Guided Backpropagation~\citep{baldassarre2019explainability}, and Grad-CAM~\citep{pope2019explainability} rely on derivative magnitudes to compute importance scores. Finally, several standalone works make use of node-based derivatives. For instance, \citet{arroyo2025vanishing} use node derivatives to draw a connection between vanishing gradients, and over-smoothing. In a different direction, \citet{keren2024sequential} propose new aggregation functions for MPNNs, designed specifically to induce non-zero mixed node derivatives.

\textbf{Learning over derivative input.}
Beyond GNNs, several recent works have explored learning directly from derivative-based inputs. \citet{xu2022signal} propose a framework that processes spatial derivatives of implicit neural representations (INRs) to modify them without explicit decoding. \citet{mitchell2021fast} introduce a learned method for fact editing in LLMs using their gradients. \citet{gelberggradmetanet} present a general architecture for learning over sets of gradients, with applications in meta-optimization, domain adaptation, and curvature estimation.

\section{Motivation}
\label{apx:motivation}
\looseness=-1
Beyond being a widely used and informative quantity in GNN analysis, MPNN derivatives can enhance expressivity. We now provide an intuition for why this occurs by drawing a connection to Subgraph GNNs.

Consider a DS-GNN model $\osan$ composed of a base MPNN $\mpnn$ followed by a downstream MPNN $\gnn$. We assume for simplicity that $\mpnn$ outputs graph-level scalars, and that the activation function $\sigma$ used in $\mpnn$ is \textit{analytic} with infinite convergence radius\footnote{Many commonly used functions, including $\sin$, $\cos$, and $\exp$, are analytic with infinite convergence radius. See Appendix~\ref{apx:proofs} for a discussion of the case where $\sigma$ is not analytic.}. That is, for every $x \in \sR$, $\sigma$ satisfies:
\begin{equation}
    \sigma(x) = \sum_{\alpha=0}^{\infty} \frac{\sigma^{(\alpha)}(0)}{\alpha!} x^\alpha.
\end{equation}
For each node $v$ of a given input graph $\gG$, we define a function $f_v: \mathbb{R} \to \mathbb{R}$ by:
\begin{equation}
    f_v(x) := \mpnn(\mA, \mX \oplus x\cdot \ve_v),
\end{equation}
which corresponds to the output of $\mpnn$ obtained by scaling the node marking feature at node $v$ by $x$. Observe that $f_v$ is analytic\footnote{For a definition of multi-dimensional analytic functions see Appendix \ref{apx:definitions}.} since $\mpnn$ is composed of analytic functions, and that by definition, $f_v(1) = \vh_v^{\text{sub}}$, i.e., the representation of the graph augmented with a mark for the node $v$. Also observe that by expanding $f_v$ around $x = 0$, we can approximate $\vh_v^{\text{sub}}$ to any desired precision using a \textbf{finite number} of derivatives, up to order $m$:
\begin{equation}
   f_v(1) \approx \sum_{\alpha=0}^{m} \frac{f_v^{(\alpha)}(0)}{\alpha!}. 
\end{equation}
Moreover, each derivative $f_v^{(\alpha)}(0)$ corresponds to a partial derivative of the output of $\mpnn$ with respect to the $v$-th coordinate of the augmented input:
\begin{equation}
    f_v^{(\alpha)}(0) = \frac{\partial^\alpha \mpnn(\mA, \mX \oplus x\cdot \ve_v)}{\partial^\alpha x}\Big|_{x=0} = \frac{\partial^\alpha \vh^{\text{out}}}{\partial^\alpha \ve_v}.
\end{equation}
This suggests that by using an encoder network to extract node features from the first $m$ derivatives of $\mpnn$ with respect to each node feature, we can effectively reconstruct $\vh_v^{\text{sub}}$. By passing these derivative-based features to the downstream GNN $\gnn$, we can approximate the behavior of $\osan$, and therefore be at least as expressive. 

\section{k-\ourmethod}
\label{apx:method}

In this section, we elaborate on the $k$-\ourmethod\ architecture, discussed in Section \ref{sec:higher_order_ourmethod}. Similar to  $1$-\ourmethod, a $k$-\ourmethod\ model, denoted $\ourgnn$, consists of a base MPNN $\mpnn$, two derivative encoders $\pool^{\text{node}}, \pool^{\text{out}}$, and a  downstream network $\gnn$. Given an input graph $\gG = (\mA, \mX)$, the computation of $\ourgnn(\gG)$ proceeds in four stages: (1) compute the output and final node representation of the input graph using $\mpnn$; (2) compute the k-indexed derivative tensors $\tD^{(T)}_k, \tD^{\text{out}}_k $; (3) extract new derivative informed node features from the derivative tensor; (4) Use these new node features for downstream processing. 

\textbf{Steps 1 \& 2.} In the first two stages, , we compute the final node representations $\vh^{(T)}$ and the output vector $\vh^{\text{out}}$ using the base MPNN $\mpnn$, along with their corresponding derivative tensors defined below (the $k$-indexed derivative tensor was defined in Section \ref{sec:higher_order_ourmethod} for he case where the input features are 1-dimensional).

\begin{definition}
    Given a graph $\gG=(\mA, \mX_0)$ with $n$ nodes, $\mX_0 \in \sR^{n\times d}$, an MPNN $\mpnn$ and some intermediate node feature representation matrix $\vh \in \sR^{n \times d'}$, the $k$-indexed derivative tensor of $\tD_k(\vh) \in \sR^{n \times n^{k}\times d'  \times m^{d \times k}} $
    is defined by:
    \begin{equation}
         \tD_k(\vh)[v, \vu, i, \valpha] = \partial^\valpha \vh_{v,i}(\mX_0) =  \frac{\partial^{|\valpha|} \vh_{v,i}}{\prod_{(j_1, j_2) \in [d]\times[k]}\partial \mX^{\valpha_{j_1, j_2}}_{u_1,j_1}, }\Big|_{\mX = \mX_0}.
    \end{equation}
    where $v \in V(\gG)$, $\vu=(u_1, \dots, u_k) \in V^k(\gG)$ , $i \in [d'], \valpha = (\valpha_{j_1, j_2})_{d,k} \in \{0, \dots, m-1 \}^{d \times k}$ and $|\valpha| = \sum \valpha_{j_1,j_2} $.

    Similarly, given a graph-level prediction vector $\vh^\text{out} \in \sR^{d'}$ the derivative tensor $\tD_k(\vh^\text{out}) \in \sR^{ n^{k}\times d' \times d^k \times m^k} $ is defined by:
    
    \begin{equation}
         \tD_k(\vh)[\vu, i, \valpha] = \partial^\valpha \vh_{i}(\mX_0).
    \end{equation}
\end{definition}

The derivative tensors $\tD^{(T)}_k = \tD(\vh^{(T)})_k$ and $\tD^{\text{out}}_k = \tD(\vh^\text{out})_k$ are computed using a message-passing-like procedure detailed in Appendix~\ref{alg:derivative_computation}, which enables both efficient computation and allows us to backpropagate through it, supporting end-to-end training of $k$-\ourmethod.

\textbf{Step 3.}  
In the third stage of our method, we extract new node features informed by the derivative tensors $ \tD^{(T)}_k, \tD^\text{out}_k,$ using the encoder networks  $\pool^{\text{node}}:   \sR^{n \times n^{k}\times d'  \times m^{d \times k}}  \to \sR^{n \times d''}$ and $\pool^{\text{out}}: \sR^{ n^{k}\times d' \times m^{d \times k}}  \to \sR^{n \times d''}$. As $\tD^{(T)}_k$ is a tensor indexed by $(k+1)$ nodes and $\tD^\text{out}_k$ is a tensor indexed by $k$ nodes, natural choices for $\pool^{\text{node}}$ and $\pool^{\text{out}}$ are $(k+1)$-IGN and a $k$-IGN \citep{maron2018invariant} designed specifically to process such data. 

Furthermore, Proposition~\ref{prop:derivative_complexity} shows that for sparse graphs or relatively shallow base MPNNs, the derivative tensor $\tD^{(T)}_k$ itself becomes sparse, with space complexity $O\left(n \cdot \min \{ n^{k}, d^{k\cdot T} \} \right)$, where $d$ is the maximum degree of the input graph. This motivates the choice of a node encoder $\pool^{\text{node}}$ that preserves this sparsity structure and exploits it for improved computational efficiency.

To enable this, we can define the encoder $\pool^{\text{node}}$ as a subclass of the general $(k{+}1)$-IGN architecture, implemented as a DeepSet~\citep{zaheer2017deep} operating independently on the derivative entries associated with each node feature $\vh^{(T)}_v$. Specifically, we define:

% To enable this,we define the encoder $\pool^{\text{node}}$ as a subclass of the general $(k{+}1)$-IGN architecture, implemented as a DeepSet~\citep{zaheer2017deep}

% which simply acts as a deep-set \citep{zaheer2017deep} architecture on the derivatives operating independently on the derivative entries associated with each node feature $\vh^{(T)}_v$. 

\begin{equation}
\label{eq:sparse_encoder}
    \pool^{\text{ds-node}}(\tD^{(T)}_k)_v = \text{DeepSet}(\{ \big( \tD^{(T)}_k[v, \vu,i, \valpha], t(v, \vu, \valpha) \big) \mid  \vu \in V^k(\gG), i \in [d'], \valpha \in [m]^{k\times d} \}).
\end{equation}

Here we assume nodes are given in index form, that is $v \in \{1, \dots, n\}$., $\vu \in \{1, \dots, n\}^k$, and $t(v, \vu, \valpha)$ is a function that encodes the derivative pattern associated with the index tuple $(v, \vu, \valpha)$, that is
\begin{equation}
t(v_1, \vu_1, \valpha_1) = t(v_2, \vu_2, \valpha_2)
\quad \Leftrightarrow \quad
\valpha_1 = \valpha_2 \text{ and } \exists \sigma \in S_n \text{ such that } v_1 = \sigma(v_2),\; \vu_1 = \sigma(\vu_2),
\end{equation}
where $S_n$ denotes the symmetric group on $n$ elements. In other words, $t$ maps each derivative index tuple to a \emph{canonical identifier} that is invariant under permutations of the node indices but sensitive to the derivative multi-index.

% that encodes the which satisfies $t(i_1, \vj_1, \valpha_1)= t(i_2, \vj_2, \valpha_2)$ if and only if $\valpha_1 = \valpha_2$, and there exists a permutation map $\sigma \in S_n$ such that $i_1 = \sigma(i_2)$, $j_1 = \sigma(j_2)$. That is $t$ encodes the derivative pattern of the index tuple $(i_1, \vj_1, \valpha_1)$.
% \begin{equation}
     
% \end{equation}

This design improves both space and time complexity, as the sets over which the DeepSet operates are typically sparse. The proof of Theorem \ref{thm:ourmethod_vs_osan} provided in appendix \ref{apx:proofs} shows that this encoder architecture is enough to be as expressive as $k$-OSAN.

Finally, we proceed the same way as 1-\ourmethod\, constructing the derivative-informed node features $\vh^{\text{der}}$ by combining information from the base MPNN, the pooled intermediate derivatives, and the output derivatives:
\begin{equation}
% \label{eq:derivative_informed_node_features}
    \vh^{\text{der}}_v =  \vh^{(T)}_v \oplus \pool^{\text{out}}(\tD^{\text{out}}_k) \oplus \pool^{\text{node}}(\tD^{\text{(T)}}_k).
\end{equation}

\textbf{Step 4.}  
This step is identical to that of 1-\ourmethod\ described in Section \ref{sec:1-ourmethod}.

\section{Derivative Computation}
\label{apx:derivative_comutation}
We now extend the derivative computation algorithm presented in Section \ref{sec:gradient_computation}, to account for $k$-mixed derivatives as well as more general MPNN architectures. Similarly to Section \ref{sec:gradient_computation}, we first split the node update procedure of an MPNN into two parts: an aggregation step:
\begin{equation}
\label{eq:mpnn_agg}
\tilde{\vh}^{(t-1)}_v = \vh^{(t-1)}_v  \oplus \agg^{(t)}\left(\left\{\vh_u^{(t-1)} : u \in \gN(v)\right\}\right),
\end{equation}

and a DeepSet step:
\begin{equation}
\label{eq:mpnn_deepset}
    \vh^{(t)}_v = \mlp(\tilde{\vh}^{(t-1)}_v).
\end{equation}

Our algorithm begins by computing the initial derivative tensor $\tD^{(0)}$ (we abuse notation and ommit the subscript $k$), and then recursively constructs each $\tilde{\tD}^{(t-1)} = \tD(\tilde{\vh}^{(t-1)})$   from $\tD^{(t-1)}$  and then $\tD^{(t)}$ from $\tilde{\tD}^{(t-1)}$. Finally, the output derivative tensor $\tD^{\text{out}}$ is obtained from $\tD^{(T)}$. See Algorithm~\ref{alg:derivative_computation} for the full procedure.

To analyze the time and memory complexity of each step, we define the derivative sparsity of a node $v$ of an input graph $\gG$ at layer $t$ of the base MPNN, as the number of non-zero derivatives corresponding to $\vh^{(t)}_v$. That is 
\begin{equation}
    s_{v,t} = \left|\left\{(\vu, i,  \valpha)\;\middle|\; \tD^{(t)}[v, \vu,i, \valpha] \neq 0 \right\}\right|.
\end{equation}
% The overall graph derivative sparsity at layer $t$ is
\begin{equation}
    s_t = \min_{v \in V(\gG)} s_{v,t}.
\end{equation}
When $s_t$ is small, the tensor $\tD^{(t)}$ can be stored efficiently in memory using sparse representations, requiring only $O(n \cdot s_t)$ space. The quantities $s_{v,t}$ and $s_t$ are leveraged in Appendix~\ref{apx:proofs} to derive concrete asymptotic bounds for the complexity of Algorithm~\ref{alg:derivative_computation}.

\textbf{Computing $\tD^{(0)}$.}  
Since $\vh^{(0)} = \mX$, the derivatives are straightforward to compute:
\begin{equation}
\label{eq:derivative_init_higher_order}
    \tD^{(0)}[v, \vu, i,  \valpha] =
    \begin{cases}
        1 & \text{if } \exists s \text{ s.t. }  v = \vu_s, \valpha_{s,i} = 1,  \sum_{s' \neq s, j \in [d]}\valpha_{s',j} = 0\\
        0 & \text{otherwise}.
    \end{cases}
\end{equation}
% \begin{equation}
% \label{eq:derivative_init_higher_order}
%     \tD^{(0)}[v, \vu, i, \vj, \valpha] =
%     \begin{cases}
%         1 & \text{if } \exists s \text{ s.t. }  v = \vu_s, i = \vj_s, \valpha_s = 1,  \sum_{s' \neq s}\valpha_{s'} = 0\\
%         0 & \text{otherwise}.
%     \end{cases}
% \end{equation}
% \begin{equation}
%     \tD^{(0)}[v, \vu, i, \vj, \valpha] =
%     \begin{cases}
%         1 & \text{if } v = u_1 = \cdots = u_k,  \text{ and } \sum \alpha_i = 1, \\
%         0 & \text{otherwise}.
%     \end{cases}
% \end{equation}
Although $\tD^{(0)}$ is high-dimensional, it is extremely sparse with $s_{0} = O(1)$. 

\textbf{Computing $\tilde{\tD}^{(t-1)}$ from $\tD^{(t-1)}$.}  
Abusing notation, and assuming that for every node $v$ in the input graph, both $\vh^{(t-1)}_v$ and $\vh^{(t-1), \text{agg}}_v = \agg^{(t)}\left(\{\vh_u^{(t-1)} : u \in \gN(v)\} \right)$ lie in $\sR^{d'}$, we observe that since $\tilde{\vh}^{(t-1)}_{v, 0{:}d'-1} = \vh^{(t-1)}_v$, the derivatives of the first $d'$ coordinates of $\tilde{\vh}$ are precisely $\tD^{(t-1)}$. In contrast, the derivatives with respect to the last $d'$ coordinates depend heavily on the choice of aggregation function $\agg$. 

However, when the aggregation function is linear, i.e., of the form
\begin{equation}
\label{eq:linear_agg}
    \agg^{(t)}\left(\{\vh_u^{(t-1)} : u \in \gN(v)\}\right) = \sum_{v' \in \gN(v)} b_{v',v} \vh^{(t-1)}_{v'},
\end{equation}
for some coefficients $b_{v',v}$ that depend only on the adjacency matrix $\mA$ of the input graph (this is the case in most widely used MPNNs), the computation of derivatives simplifies. Since differentiation commutes with linear operations, the computation in Equation \ref{eq:linear_agg} carries over, yielding
\begin{equation}
\label{eq:derivative_linear_agg}
    \tD(\vh^{(t-1), \text{agg}})[v, \dots] = \sum_{v' \in \gN(v)} b_{v',v} \tD^{(t-1)}[v', \dots].
\end{equation}

Aggregating the derivatives of neighboring nodes mirrors the structure of message passing, which endows it with several beneficial properties. First, using Equation~\ref{eq:derivative_linear_agg}, the tensor $\tilde{\tD}^{(t-1)}[v, \dots]$ can be computed in time $O(d\cdot s_{ t-1})$ where $d$ is the maximal degree of the input graph\footnote{We slightly abuse notation by using $d$ to denote both the input feature dimension and other quantities; the meaning should be clear from context.}: for each node, we aggregate $d$ neighbor derivative vectors, each containing at most $s_{t-1}$ non-zero entries. This means that the total computation time of this step is $O(d\cdot n \cdot  s_{t-1})$.

The argument above also implies that the total number of non-zero elements in $\tilde{\tD}^{(t-1)}[v, \dots]$ is bounded by $O\left( \min\{d\cdot s_{t-1},\; n^k\}\right)$. Thus, the above update leverages the sparsity of the graph to achieve efficiency in both space and time complexity.

% To show this, we first define the derivative sparsity of a node $v$ at layer $t$ as
% \begin{equation}
%     s_{v,t} = \left|\left\{(\vu, i, \vj, \valpha)\;\middle|\; \tD^{(t)}[v, \vu,i, \vj, \valpha] \neq 0 \right\}\right|,
% \end{equation}
% and define the overall graph derivative sparsity at layer $t$ as
% \begin{equation}
%     s_t = \min_{v \in V(\gG)} s_{v,t}.
% \end{equation}

% Using Equation~\ref{eq:derivative_linear_agg}, the tensor $\tilde{\tD}^{(t-1)}$ can be computed in time $O(dns_{t-1})$: for each of the $n$ nodes, we aggregate $d$ neighbor derivative vectors, each containing at most $s_v^{t-1}$ non-zero entries. Moreover, the total number of non-zero elements in $\tilde{\tD}^{(t-1)}$ is bounded by $O\left(n \cdot \min\{ds_{t-1},\; n^k\}\right)$. Thus, the above update leverages the sparsity of the graph to achieve efficiency in both space and time complexity.

% Recalling that the initial derivative tensor $\tD^{(0)}$ has $O(n)$ non-zero entries, th

% Assuming that the maximal number of nonzeo entries for any derivative t

\textbf{Computing $\tD^{(t)}$ from $\tilde{\tD}^{(t-1)}$.}

We begin by assuming that the MLP in Equation~\ref{eq:mpnn_deepset} has depth $1$, i.e.,
\begin{equation}
    \mlp^{(t)}(\vx) = \sigma(\mW^{(t)} \cdot \vx +\vb^{(t)}).
\end{equation}
In the general case where the MLP has depth $l$, the procedure described below is applied recursively $l$ times.

We define the intermediate linear activation as
\begin{equation}
\label{eq:learned_linear}
    \vh^{(t-1), \text{Lin}}_v = \mW^{(t)} \cdot \tilde{\vh}^{(t)}_v + \vb^{(t)},
\end{equation}
and describe how to compute $\tD(\vh^{(t-1), \text{Lin}})$ from $\tilde{\tD}^{(t-1)}$, followed by the computation of $\tD^{(t)}$ from $\tD(\vh^{(t-1), \text{Lin}})$.

First, since the update in Equation~\ref{eq:learned_linear} is affine, we can drop the bias term and get:
\begin{equation}
\label{eq:agg_to_linear_deriv_update}
    \tD(\vh^{(t-1), \text{Lin}})[v, \vu, i,  \alpha] = \sum_{i'} \mW^{(t)}_{i,i'} \cdot \tilde{\tD}^{(t-1)}[v, \vu, i',  \alpha].
\end{equation}

Second, notice that
\begin{equation}
\label{eq:deepset_in_deriv_comp}
    \vh^{(t)}_v = \sigma(\vh^{(t-1), \text{Lin}}_v).
\end{equation}
We can use the Faà di Bruno’s formula (see e.g. \citep{hardy2006combinatorics}) which generalizes the chain rule to higer-order derivatives, for our last step. Faà di Bruno’s formula  states that for a pair of functions $g: \sR^n \to \sR$, $f: \sR \to \sR$, \( y = g(x_1, \ldots, x_n) \) the following holds, regardless of whether the variables $ x_1, \dots, x_n $ are all distinct, identical, or grouped into distinguishable categories of indistinguishable variables:

\begin{equation}
\label{eq:fa_di_bruno}
    \frac{\partial^n}{\partial x_1 \cdots \partial x_n} f(y) = 
\sum_{\pi \in \Pi} f^{(|\pi|)}(y) \cdot \prod_{B \in \pi} \frac{\partial^{|B|} y}{\prod_{j \in B} \partial x_j}
\end{equation}

where:
\begin{itemize}
    \item $\Pi$  denotes the collection of all partitions of the index set $ \{1, \ldots, n\} $,
    \item The notation  $B \in \pi$  indicates that  $B$  is one of the subsets (or "blocks") in the partition  $\pi$ ,
    \item For any set  $A$ , the notation  $|A|$  represents its cardinality. In particular,  $|\pi|$ is the number of blocks in the partition, and  $|B|$  is the number of elements in the block  $B$.
\end{itemize}

Equations \ref{eq:deepset_in_deriv_comp} and \ref{eq:fa_di_bruno} Imply that we are able to compute $\tD^{(t)}[v, \vu, \dots]$ based only on $\tD(\vh^{(t-1), \text{Lin}})[v, \vu, \dots]$  and the derivatives of $\sigma$ at the point $\vh^{(t-1), \text{Lin}}_v$. Combining this update with Equation  \ref{eq:agg_to_linear_deriv_update} results in a way to compute $\tD^{(t)}$ from $\tilde{\tD}^{(t-1)}$.

Importantly, the update above computes each entry of $\tD^{(t)}[v, \vu, \dots]$ using only the corresponding entries of $\tilde{\tD}^{(t-1)}[v, \vu, \dots]$, i.e., those associated with the same node tuple $(v, \vu)$. Moreover, from Equations~\ref{eq:agg_to_linear_deriv_update} and~\ref{eq:fa_di_bruno}, it follows that if $\tilde{\tD}^{(t-1)}[v, \vu, \dots] = \mathbf{0}$, then $\tD^{(t)}[v, \vu, \dots] = \mathbf{0}$ as well. This implies that like $\tilde{\tD}^{(t-1)}$ the number of non-zero entries in each $\tD^{(t)}[v, \dots]$ is also bounded by $O\left(\min\{d\cdot s_{ t-1},\; n^k\}\right)$. Consequently, the update—performed only over the non-zero entries of $\tilde{\tD}^{(t-1)}$—has a runtime complexity of $O\left(n \cdot \min\{d\cdot s_{t-1},\; n^k\}\right)$.

% for a pair of functions $g: \sR^n \to \sR$, $f: \sR \to \sR$ it holds that
% \begin{equation}
%     \frac{\partial^{\valpha_1+ \dots + \valpha_k} f \circ g}{\partial x_1^{\valpha_1}, \dots, \partial x_k^{\valpha_k}} = \sum_{\pi \in \Pi} f^{(\pi)}(g(x)) \cdot \prod_{\vbeta \in \pi} \frac{\partial^{\vbeta_1+ \dots + \vbeta_k}}{\partial x_1^{\vbeta_1}, \dots, \partial x_k^{\vbeta_k}}
% \end{equation}
% where

% \begin{itemize}
%     \item $\pi$ runs through the set $\Pi$ of all partitions of $[\valpha_1+ \cdots+\valpha_k]$.
%     \item $\vbeta \in \pi$ means the variable \( B \) runs through the list of all the "blocks" of the partition \( \pi \), and \( |A| \) denotes the cardinality of the set \( A \) (so that \( |\pi| \) is the number of blocks in the partition \( \pi \), and \( |B| \) is the size of the block \( B \)).

% \end{itemize}

% We can use  Faà di Bruno’s formula \citep{} which generalizes the chain rule to higer-order derivatives to obtain

% \begin 

\textbf{Computing $\tD^{\text{out}}$ from $\tD^{(T)}$.}

Recall that
\begin{equation}
    \vh^{\text{out}} = \agg_{\text{fin}}\left(\left\{\vh_v^{(T)} \mid v \in V(\gG)\right\}\right),
\end{equation}
where \(\agg_{\text{fin}}\) denotes the final aggregation over node embeddings. This operation can be treated analogously to the node update step: For most common MPNNs, it decomposes into a linear aggregation followed by an MLP. Consequently, the derivative tensor \(\tD^{\text{out}}\) can be computed from \(\tD^{(T)}\) using the same primitives described above.

\begin{algorithm}[H]
    \caption{Efficient Computation of Derivative Tensors}\label{alg:derivative_computation}
    \begin{algorithmic}[1]
        \Require{ Graph $\gG = (\mA, \mX)$, base GIN $\mpnn$ with $T$ layers}
        \State $\vh^{(0)} \gets \mX$ \Comment{node feature init.}
        \State $\tD^{(0)} \gets \tD(\mX)$ \Comment{deriv. init through Eq \ref{eq:derivative_init_higher_order}.}
        \For{$t = 1$ to $T$}
            \State $\tilde{\vh}^{(t-1)}_v  \gets \vh^{(t-1)}_v\oplus \left(\sum_{v' \in \gN(v)} b_{v',v} \vh^{(t-1)}_{v'} \right)$ \Comment{linear node agg.}
            \State $\tilde{\tD}^{(t-1)}[v, \dots]  \gets \tD^{(t-1)}[v, \dots] \oplus \left( \sum_{v' \in \gN(v)} b_{v',v} \tD^{(t-1)}[v', \dots] \right)$ \Comment{deriv. agg.}
        
            \State $\vh^{(t)}_v = \mlp^{(t-1)}(\tilde{\vh}^{(t-1)})_v$ \Comment{DeepSet update}
            \State $\tD^{(t)}[v, \dots] = \text{get-der}(\mlp^{(t-1)}, \tilde{\tD}^{(t-1)}[v, \dots], \tilde{\vh}^{(t-1)}_v  )$ \Comment{deriv. DeepSet update.}
        \EndFor
        
        $\tD^\text{out} = \text{get-out-der}(\tD^{(T)})$  \Comment{extract output deriv.}
        \State \Return $\tD^{(T)},\tD^{\text{out}}$
    \end{algorithmic}
\end{algorithm}

\section{Extended Theoretical Analysis }
\label{apx:theory}

\subsection{Definitions}
\label{apx:definitions}
Before delving into the proofs, we begin by formally defining several key concepts used throughout the analysis:

\begin{definition}[Analytic Function]
\label{def:analytic_function}
A function $f: \mathbb{R}^n \to \mathbb{R}$ is said to be \emph{analytic} at $\vx_0 \in \sR^{n}$ if for some $R \in \sR^n$ it holds that for all $|X|$ < R:

\begin{equation}
   f(\vx) = 
\sum_{|\valpha| = 0}^{\infty} 
\frac{1}{\valpha!} 
\partial^\valpha f(\vx_0) 
(\vx - \vx_0)^\valpha
\end{equation}

where $\valpha \in \sN^n$ and we use the following notation:
\begin{itemize}
    \item $|\valpha| = \alpha_1 + \alpha_2 + \cdots + \alpha_n$, 
    \item $\valpha! = \alpha_1! \cdot \alpha_2! \cdots \alpha_n!$,
    \item $(\vx )^\valpha = (x_1)^{\alpha_1} \cdot (x_2 )^{\alpha_2} \cdots (x_n)^{\alpha_n}$
    \item $\partial^\valpha f(\mathbf{a}) = 
    \frac{\partial^{|\valpha|} f}{\partial x_1^{\alpha_1} \partial x_2^{\alpha_2} \cdots \partial x_n^{\alpha_n}} \Big|_{\mathbf{x} = \mathbf{a}}.$
\end{itemize}

The largest such $R$ is called the radius of convergence. A function $f: \sR^n \to \sR^m$ is analytic if all functions $f_1, \dots, f_m$ are analytic.

\end{definition}

\begin{definition}[$k$-OSAN]
\label{def:osan}
A $k$-OSAN model $\osan$ consists of a base MPNN $\mpnn$ that produces updated node features (as opposed to directly outputting a graph-level prediction), followed by a downstream MPNN $\gnn$ that aggregates these features to produce a final graph-level output. Given an input graph $\gG = (\mA, \mX)$, the output $\osan(\gG)$ is computed in four stages:

\textbf{Step 1}: Construct a bag of subgraphs $\mathcal{B}_\gG = \{\gS_{\vu} \mid \vu \in V^k(\gG) \}$ each of the form $\gS_{\vu} = (\mA, \mX \oplus \ve^{\vu})$. Here $\ve^{\vu} \in \sR^{n \times k}$ is a "node marking"\footnote{Although alternative methods for initializing node features in subgraphs have been proposed, they offer the same expressive power. We therefore focus on the simple approach used here.} feature matrix assigning a unique identifier to each node $u_1, \dots u_k$. That is:

\begin{equation}
\label{eq:osan_marking}
    \ve^{\vu}_{v,j} = \begin{cases}
         1 & v = u_j \\
        0 & \text{else.}
    \end{cases}
\end{equation}

\textbf{Step 2}:

Compute the $(k+1)$-node indexed tensor:
\begin{equation}
    \tH[v, \vu] = \mpnn(\gS_{\vu})_v
\end{equation}

\textbf{Step 3}:
use a set aggregation function to produce new node features:

\begin{equation}
\label{eq:osan_aggregation}
    \vh^{\text{sub}}_v = \agg(\{\tH[v, \vu] \mid \vu \in V^k(\gG)  \}).
\end{equation}

\textbf{Step 4}:
Compute the final output through:
\begin{equation}
    \osan(\gG) = \gnn(\mA, \vh^{\text{sub}}).
\end{equation}

For $k=1$, $k$-OSAN are also reffered to as DS-GNNs. 

\end{definition}

% \begin{definition}[$k$-OSAN]
% A $k$-OSAN model $\osan$ consists of a base MPNN $\mpnn$ that produces updated node features (as opposed to directly outputting a graph-level prediction), followed by a downstream MPNN $\gnn$ that aggregates these features to produce a final graph-level output. Given an input graph $\gG = (\mA, \mX)$, the output $\osan(\gG)$ is computed in four stages:

% \textbf{Step 1}: Construct a bag of subgraphs $\mathcal{B}_\gG = \{\gS_{\vy} \mid \vu \in V^k(\gG) \}$ each of the form $\gS_{\vu} = (\mA, \mX \oplus \ve_{\vu})$. Here $\ve_{\vu}$ assigns a unique identifier to each node $u_1, \dots u_k$:

% \begin{equation}
% \label{eq:osan_marking}
%     \ve_{\vu}[v] = \begin{cases}
%         \va_i & v = u_i \\
%         \textbf{0} & \text{else,}
%     \end{cases}
% \end{equation}
%     where $\va_1, \dots , \va_k$ are unique learned vecotrs.

% \textbf{Step 2}:

% Compute the $(k+1)$-node indexed tensor:
% \begin{equation}
%     \tH[v, \vu] = \mpnn(\gS_{\vu})_v
% \end{equation}

% \textbf{Step 3}:
% use a set aggregation function to produce new node features:

% \begin{equation}
% \label{eq:osan_aggregation}
%     \vh^{\text{sub}}_v = \agg(\{\tH[v, \vu] \mid \vu \in V^k(\gG)  \}).
% \end{equation}

% \textbf{Step 4}:
% Compute the final output through:
% \begin{equation}
%     \osan(\gG) = \gnn(\mA, \vh^{\text{sub}}).
% \end{equation}

% \end{definition}

\begin{definition}[Random Walk Structural Encoding]
\label{def:rwse}
For a graph$\gG = (\mA, \mX)$ , the Random Walk Structural Encoding (RWSE) with $L$ number of steps is defined as
\begin{equation}
    \vh^{\text{rwse}} = \bigoplus_{l=1}^L \mathrm{diag}(\tilde{\mA}^l),
\end{equation}
where $\tilde{\mA} = \mA \cdot \mathrm{Diag}(\deg(u_1)^{-1}, \dots, \deg(u_n)^{-1})$ is the row-normalized adjacency matrix,  
$\mathrm{diag}(\cdot)$ denotes the vector of diagonal entries of a matrix, and $\mathrm{Diag}(\vv)$ denotes the diagonal matrix with vector $\vv$ on its diagonal.
\end{definition}

\subsection{Proofs}
\label{apx:proofs}

\textbf {Theorem \ref{thm:ourmethod_vs_osan}.}
We begin by formally stating and proving Theorem \ref{thm:ourmethod_vs_osan}, splitting it into a Theorem and a corollary. 

\begin{theorem}[$k$-\ourmethod\ is as expressive as $k$-OSAN]
\label{thm:ourmethod_vs_osan_formal}

Let $\{ \gG^i =(\mA^i, \mX^i) \mid i \in [l] \}$ be a finite set of graphs, and let $\osan$ be a $k$-OSAN model. for any $\epsilon > 0$, there exists a $k$-\ourmethod model $\ourgnn$ such that for each $i \in [l]$ 
\begin{equation}
   \left| \osan(\gG_i) - \ourgnn(\gG_i)  \right|< \epsilon.
\end{equation}
    
\end{theorem}

\begin{corollary}
\label{cor:ourmethod_vs_k_wl}
    There exist non-isomorphic graphs that are indistinguishable by the folklore $k$-WL test but are distinguishable by $k$-\ourmethod. Additionally,  $k$-\ourmethod\ is able to compute the homomorphism count of $k$-apex forest graphs.
\end{corollary}

\begin{proof}
    We begin the proof by making a few simplifying assumptions on $\osan$, which we can do without loss of generality. We begin by assuming that all activation functions used in $\mpnn$, the base MPNN of $\osan$, are analytic with infinite radius of convergence (e.g., $\exp(x)$, $\sin(x)$). This assumption can be made without loss of generality: Since MLPs with non-polynomial analytic activations are universal approximators, each MLP in $\mpnn$ can be replaced with one using an analytic activation function that approximates the original to arbitrary precision. Furthermore, since the composition of analytic functions with infinite convergence radius remains analytic with infinite convergence radius, it follows that $\mpnn$—as a composition of affine transformations and activation functions—is itself analytic.

    Secondly, we assume that the final node representations produced by $\mpnn$ are one-dimensional. This assumption can be made without loss of generality: we can append a final MLP to $\mpnn$ that compresses each node's feature vector to a scalar, and prepend an MLP to the downstream network $\gnn$ that reconstructs the original feature dimension. This effectively amounts to inserting an autoencoder, where the encoder is a final pointwise update in $\mpnn$ and the decoder is an initial pointwise update in $\gnn$. Since this architecture can approximate the original architecture to arbitrary precision, we may assume without loss of generality that $\mpnn$ produces 1-dimensional node embeddings.

    Additionally, we can assume the input node feature matrices $\mX^i, i \in [l]$ are also all 1-dimensional. This follows from the same argument as above. 
    
    Finally, we consider $k$-\ourmethod models that only use the node derivative tensor $\tD^{(T)}$, which we use to extract node features through an IGN encoder $\pool$, and disregard $\tD^{\text{out}}$.

    We prove the theorem in three steps
    \begin{enumerate}
        \item We show that an intermediate representation of $\pool(\tD^{(T)})$ can encode the tensor $\tH_{v, \vu}$ that is produced at stage (2) of the forward pass of $\osan$ (see Definition \ref{def:osan}).
        \footnote{In cases where $\vu$ has repeated entries $\vu_{j_1} = \vu_{j_2}$, we only consider values of $\mY$ for which $\mY_{j_1, :} = \mY_{j_2, :}$.}

        \item we show that $\pool(\tD^{(T)})$ can approximate the node feature matrix $\vh^{\text{sub}}$ produced at stage (3) of the forward pass of $\osan$.

        \item We show $\ourgnn$ can approximate $\osan$ to finite precision.
        
    \end{enumerate}
    \textbf{Step 1}
    For each k-tuple of nodes $\vu \in V^k(\gG^i)$ and matrix $\mY \in \sR^{k \times k}$,  we define the node feature matrix $\text{broad}_\vu(\mY) \in \sR^{n\times k}$ (here we abuse notation and not include the graph index $i$) by 
    \begin{equation}
            \text{broad}_\vu(\mY)_{v,i} = \begin{cases}
            y_{j,i} & v = \vu_j\\
            0 & \text{else.}
        \end{cases}
    \end{equation}
    Additionally, for each  node $v \in \gG^i$, we define 
    $f^i_{v, \vu}:\sR^{k\times k}: \to \sR$
    \begin{equation}
        f^i_{v, \vu}(\mY) = \mpnn(\mA, \mX \oplus \text{broad}_\vu(\mY)).
    \end{equation}
    
Finally, define $\mY^\vu \in \sR^{k \times k}$ by 
\begin{equation}
    \mY^\vu_{i,j} = \ve^\vu_{u_i, j}.
\end{equation}
where $ \ve^\vu$ is the node marking node feature matrix introduced in Definition \ref{def:osan} (That is, $\text{broad}_\vu(\mY^\vu)= \ve^\vu$ ).

First, it is easy to see that
\begin{equation}
    f^i_{v, \vu}(\mY^\vu) = \mpnn(\gS^i_\vu)_v=\tH^i[v, \vu], 
\end{equation}
where $\gS^i_\vu$ and $\tH^i[v, \vu]$ are introduced in Definition \ref{def:osan}.

Second, as $\mpnn$ is analytic with infinite convergence radius, the functions $f^i_{v, \vu}$ are all analytic with infinite convergence radi, and so
\begin{equation}
    f^i_{v, \vu}(\mY^\vu) = 
\sum_{|\valpha| = 0}^{\infty} 
\frac{1}{\valpha!} 
\partial^\valpha f^i(\mathbf{0}) 
(\mY^\vu)^\valpha.
\end{equation}

Here, $\valpha \in \{0 , \dots, m-1 \}^{k \times k} $ and $(\mY^\vu)^\valpha = \prod (\mY^\vu_{j_1, j_2})^{\valpha_{j_1, j_2}}$. Since we are concerned with a finite number of graphs, for any $\epsilon > 0$ we can choose an integer $I$ such that for all graphs $\gG^i$ and all $v \in V(\gG^i)$, $\vu \in V^k(\gG^i)$ it holds that:
\begin{equation}
    |f^i_{v, \vu}(\mY^\vu) - \sum_{|\valpha| = 0}^{I} 
\frac{1}{\valpha!} 
\partial^\valpha f^i_{v, \vu}(\mathbf{0}) 
(\mY^\vu)^\valpha| < \epsilon.
\end{equation}

Defining 
\begin{equation}
    \tilde{\tH}[^iv, \vu] = \sum_{|\valpha| = 0}^{I} 
\frac{1}{\valpha!} 
\partial^\valpha f^i_{v, \vu}(\mathbf{0}) 
(\mY^\vu)^\valpha,
\end{equation}
we get that 
\begin{equation}
\label{eq:derivative_aprox}
    \tilde{\tH}^i \approx \tH^i
\end{equation}

Additionally, from the definition of $f^i_{v, \vu}$ it holds that the derivatives of $f^i_{v, \vu}$ at zero correspond to the entries of the $k$-indexed derivative tensor of $\gG^i$, denoted by $\tD^{(T),i}$\footnote{here we abuse notation and omit the subscrit $k$ in $\tD_k$.}, that is

\begin{equation}
    \partial^\valpha f^i_{v, \vu}(\mathbf{0})  = \tD^{(T),i}[v, \vu, \valpha],
\end{equation}
where we only take derivatives with respect to the last $k$ feature dimensions, which correspond to the "marking vectors"

Moreover, let $t$ be a function such that $t(v, \vu, \valpha)$ encodes the derivative pattern associated with the index tuple $(v, \vu, \valpha)$. That is, $t$ satisfies:
\begin{equation}
t(v_1, \vu_1, \valpha_1) = t(v_2, \vu_2, \valpha_2)
\quad \Leftrightarrow \quad
\valpha_1 = \valpha_2 ;\text{ and }; \exists \sigma \in S_n \text{ such that } v_1 = \sigma(v_2),; \vu_1 = \sigma(\vu_2),
\end{equation}
where $S_n$ is the symmetric group on $n$ elements. The value of $\frac{1}{\valpha!}(\mY^{\vu})^{\valpha}$ is determined entirely by $t(v, \vu, \valpha)$, and thus can be recovered from it.

This implies that the DeepSet encoder $\pool^\text{ds-node}$ defined in Equation \ref{eq:sparse_encoder} in Appendix \ref{apx:method}, can have an intermediate layer $L$ such that

\begin{equation}
\label{eq:taylor_aprox}
    L(\tD^{(T),i}) = \tilde{\tH}^i.
\end{equation}

Here $L$ simply  multiples each entry $\tD^{(T),i}[v, \vu, \valpha]$ by $\frac{1}{\valpha!}(\mY^{\vu})^{\valpha}$, followed by summing over the $\valpha$ indices. Thus, by Equations~\ref{eq:derivative_aprox} and~\ref{eq:taylor_aprox}, choosing $\mpnn$ as the base MPNN of a $k$-\ourmethod\ model allows us to approximate each $\tH^i$ to arbitrary precision using its derivatives. Note that since the DeepSet encoder is a restricted instance of a $(k+1)$-IGN encoder, it can achieve the same effect.

\textbf{Step 2}
Equation~\ref{eq:osan_aggregation} shows that the node features $\vh^{\text{sub},i}_v$ are constructed by applying a set-wise aggregation function over the set $\{ \tH^i[v,\vu] \mid \vu \in V^k(\gG^i) \}$. Any continuous set-wise aggregation function can be approximated to arbitrary precision by a DeepSet architecture \citep{zaheer2017deep} (see \cite{segol2019universal} for proof). Moreover, any DeepSet model applied in parallel over the first index of a $k$-indexed node tensor can be exactly implemented by a $k$-IGN, since each layer in such a model consists of a linear equivariant transformation followed by a pointwise nonlinearity. Thus, we can construct our encoder $\pool$ to first approximate the mapping $\tD^i \to \tilde{\tH}^i$ as an intermediate representation, and then approximate the subsequent mapping $\tilde{\tH}^i \to \vh^{\text{sub},i}$. Since both approximations can be made to arbitrary precision, this completes the proof of the claim in Step 2.

\textbf{Step 3}
The final forward step—applying a downstream GNN to the updated node features to produce a graph-level representation—is identical in both OSAN and \ourmethod. Therefore, by choosing the downstream GNN in the $k$-\ourmethod\ model to match that of $\osan$, the proof is complete.

\end{proof}

Using  Theorem \ref{thm:ourmethod_vs_osan_formal} we now prove corollary \ref{cor:ourmethod_vs_k_wl}

\begin{proof}[Proof of corollary \ref{cor:ourmethod_vs_k_wl}]
  It was shown by \citet{qian2022ordered} that $k$-OSAN models can distinguish between graphs that are indistinguishable by the $k$-WL test. As Theorem~\ref{thm:ourmethod_vs_osan_formal} establishes that our method can approximate any $k$-OSAN model to arbitrary precision, it follows that $k$-\ourmethod~ can do the same. Similarly, \citet{zhang2024beyond} showed that $k$-OSAN models can compute homomorphism counts of $k$-apex forests—graphs in which the removal of at most $k$ nodes yields a forest. Therefore, by Theorem~\ref{thm:ourmethod_vs_osan_formal}, Corollary~\ref{cor:ourmethod_vs_k_wl}  follows.
\end{proof}

\textbf {Theorem \ref{thm:ourmethod_vs_rwse}.}
We now formally state and prove Theorem \ref{thm:ourmethod_vs_rwse}
\begin{theorem}[\ourmethod\ is strictly more expressive than RWSE+MPNN]
\label{thm:ourmethod_vs_rwse_formal}
For any MPNN $\gnn$ augmented with random walk structural encodings (see Definition~\ref{def:rwse}), there exists a $1$-\ourmethod\ model $\ourgnn$ that uses ReLU activations and only first-order derivatives such that, for every graph $\gG$, it holds that
\begin{equation}
    \gnn(\gG) = \ourgnn(\gG).
\end{equation}

Moreover, there exist a pair of graphs $\gG^1$ and $\gG^2$ such that for every RWSE-augmented MPNN $\gnn$,
\begin{equation}
\label{eq:recreate_rwse}
    \gnn(\gG^1) = \gnn(\gG^2),
\end{equation}
yet there exists a $1$-\ourmethod\ model $\ourgnn$, using ReLU activations and only first-order derivatives, such that
\begin{equation}
    \ourgnn(\gG^1) \neq \ourgnn(\gG^2).
\end{equation}
\end{theorem}

\begin{figure*}[t]
    \centering
    \includegraphics[width=\textwidth]{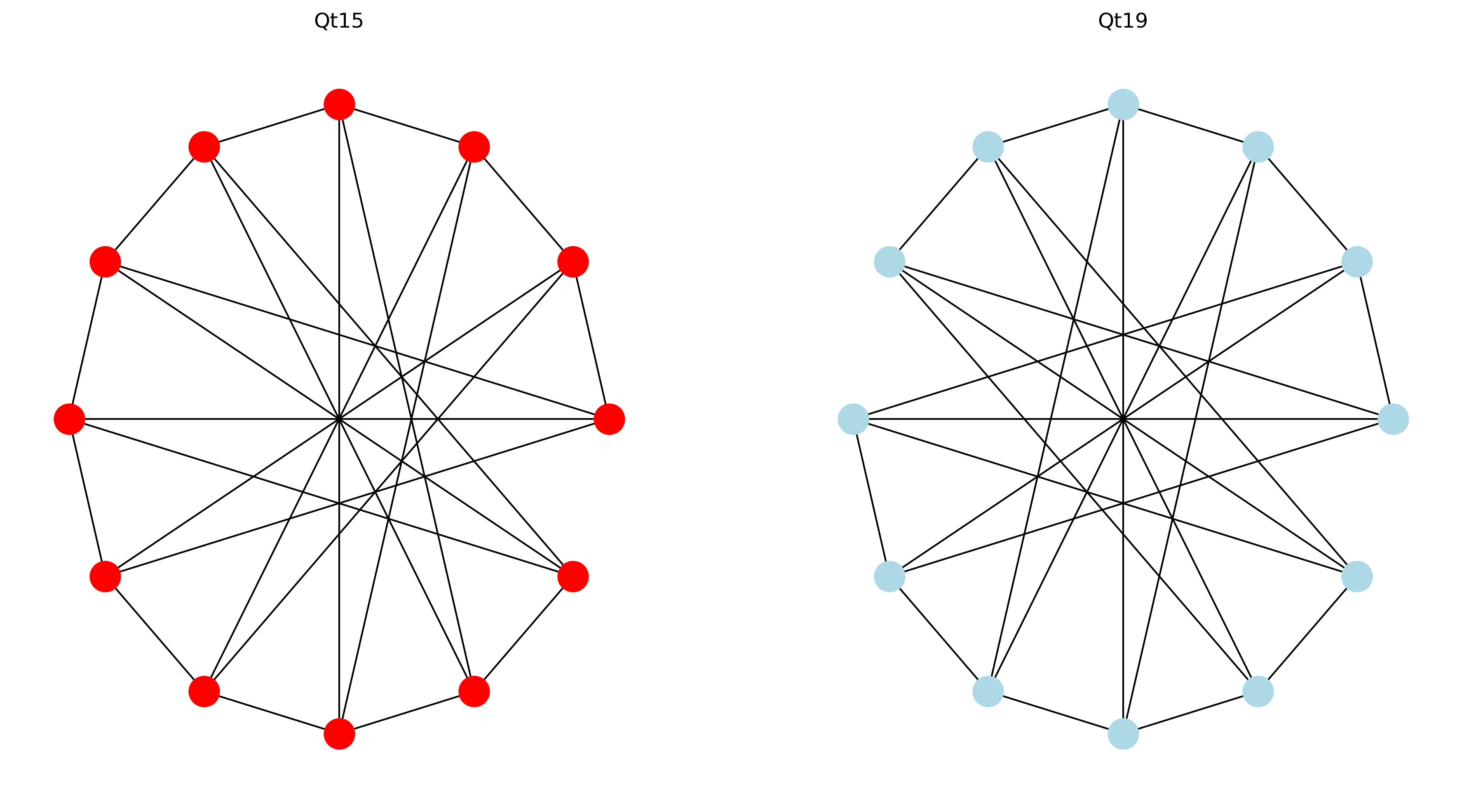}
    \vspace{-10pt}
    \caption{Two quartic vertex-transitive graphs that cannot be distinguished by MPNNs augmented with RWSE, but can be distinguished using \ourmethod.
}\label{fig:transitive_graphs}
\end{figure*}

\begin{proof}
   To prove the first part of the theorem for RSWE with  $L$ number of steps, we begin with a simple preprocessing step. For each input graph $\gG = (\mA, \mX)$ with node feature matrix $\mX \in \mathbb{R}^{n \times d}$, we define an extended feature matrix $\bar{\mX} = \mX \oplus \mathbf{1}_{L} \in \mathbb{R}^{n \times (d+L)}$ by padding $\mX$ with a constant vecotr of lenght $L$ and value $1$. We then pass $\bar{\mX}$ to the \ourmethod\ model instead of the original $\mX$.

It is now sufficient to construct a base MPNN $\mpnn$ with $T$ layers such that:
\begin{enumerate}

    \item The first $d$ coordinates of each final node embedding satisfy $\vh^{(T)}_{v,0:d-1} = \mX_v$, i.e., the original features are preserved.
    \item The first-order derivatives of the remaining $L$ coordinates (i.e., for indices $l = d, \dots, d+L-1$) satisfy:
    \begin{equation}
        \frac{\partial \vh^{(T)}_{v,l}}{\partial \bar{\mX}_{v,l}} = \tilde{\mA}^l_{v,v},
    \end{equation}
    
    where $\tilde{\mA}$ is the row-normalized adjacency matrix (see Definition~\ref{def:rwse}).
\end{enumerate}

If the above conditions hold, we can choose an MLP  such that 

\begin{equation}
    \vh^{\text{der}}_v = \mlp \left( \vh^{(T)}_v \oplus \pool^{(T)}(\tD^{(T)})_v \oplus \pool^{\text{out}}(\tD^{\text{out}}) \right) = \mX_v \oplus \vh^{\text{rwse}}_v
\end{equation}
and choose our downstream model to be exactly $\gnn$ thus satisfying Equation \ref{eq:recreate_rwse}.

We now construct an MPNN $\mpnn$ that satisfies both of the conditions above.

For each layer $t = 0, \dots, T-1$, we define the update rule of $\mpnn$ to act separately on two parts of the node feature vector: the first $d + t$ coordinates and the remaining $T - t$ coordinates. Specifically:

\begin{equation}
\label{eq:vs_osan_first_part}
    \vh^{(t+1)}_{v, 0:d +t -1} = \vh^{(t)}_{v, 0:d +t -1},
\end{equation}

\begin{equation}
\label{eq:vs_osan_second_part}
    \vh^{(t+1)}_{v, d+t: T -1} = \text{ReLU}\left(\frac{1}{\text{deg(v)}} \sum_{u \in \gN(v)}\vh^{(t)}_{u, d+t:T -1} \right).
\end{equation}.

We set the number of layers in $\mpnn$ to be $T = L$, where $L$ is the number of random walk steps used in the original RWSE encoding.

First, the proposed update rule is straightforward to implement within the MPNN framework defined by Equation~\ref{eq:mpnn}, as it follows a standard message-passing structure.

Second, Equation~\ref{eq:vs_osan_first_part} guarantees that the first $d$ coordinates of each node's feature vector remain unchanged throughout the layers, thereby satisfying condition (1).

Finally, Equations~\ref{eq:vs_osan_first_part} and~\ref{eq:vs_osan_second_part} together imply that for all $t = 0, \dots, T-1$, it holds that

\begin{equation}
    \vh^{(T)}_{v, t} = \tilde{A}^t \cdot \mathbf{1}_T.
\end{equation}
This implies that 

\begin{equation}
    \frac{\partial \vh^{(T)}_{v, t}}{\partial \bar{\mX}_{v, t}} = \tilde{\mA}^t_{v,v}
\end{equation}

and so condition (2) holds, completing the first part of the proof.

For the second part of the proof, following the notation of \citep{read1998atlas}, let $\gG_1$ and $\gG_2$ denote the quartic vertex-transitive graphs $\text{Qt}15$ and $\text{Qt}19$, respectively (see Figure \ref{fig:transitive_graphs}) with all initial node features equal to $\mathbf{1}_2 \in \sR^2$. Here, "vertex-transitive" means that for any pair of nodes, there exists a graph automorphism mapping one to the other, and "quartic" indicates that the graphs are 4-regular. This pair of graphs was shown in \citep{southern2025balancing} to be indistinguishable by MPNNs augmented with RWSE. Therefore, to conclude the proof, it suffices to construct an \ourmethod\ model that can distinguish between them. 

We define the base MPNN $\mpnn$ to consist of two layers, specified as follows. The first layer performs standard neighbor aggregation:

\begin{equation}
    \vh^{(1)}_v = \text{ReLU}\left( \sum_{u \in \gN(v)} \vh^{(0)}_u \right).
\end{equation}

In the second layer, the first coordinate of each node feature is preserved, while the second coordinate is updated via another aggregation step. That is

\begin{equation}
\label{eq:more_than_rwse_first_part}
    \vh^{(2)}_{v, 0} = \vh^{(1)}_{v, 0},
\end{equation}

\begin{equation}
\label{eq:beating_rwse_second_part}
    \vh^{(2)}_{v, 1} = \text{ReLU}\left( \sum_{u \in \gN(v)}\vh^{(1)}_{u, 1} \right).
\end{equation}.

from the same argument as the first part of the proof, we get that for $\mpnn$, it holds that 

\begin{equation}
        \frac{\partial \vh^{(2)}_{v,l}}{\partial \mX_{u,l}} = \mA^l_{v,u},
    \end{equation}

Note that for a pair of nodes $u, v$, we have $\mA_{u,v} = 1$ and $\mA^2_{u,v} = 0$ if and only if $u \in \gN(v)$ and, for all $u' \in \gN(v)$, it holds that $u \notin \gN(u')$—that is, there is no path of length exactly 2 from $u$ to $v$. A straightforward computation shows that no such node pairs exist in $\gG^1$, whereas several such pairs appear in $\gG^2$. Thus, the set of off-diagonal derivative vectors $\{ \tD^{(2),2}[u, v, 1, :] \mid u \neq v \in V(\gG^2) \}$ contains values that do not appear in the corresponding set $\{ \tD^{(2),1}[u, v, 1, :] \mid u \neq v \in V(\gG^1) \}$.

We define the node derivative encoder $\pool^{\text{node}}$ as a 2-IGN model that operates in two steps. First, it constructs a filtered tensor $\bar{\tD}$ by zeroing out all entries of $\tD^{(2)}$ that are not off-diagonal or are farther than $\epsilon = \frac{1}{4}$ from the vector $(1, 0)$. Then, it performs row-wise summation over $\bar{\tD}$ to produce node features:
\begin{equation}
    \pool^{\text{node}}(\tD^{(2)})_v = \sum_{u \in V(\gG)} \bar{\tD}[v, u].
\end{equation}

It follows that $\pool^{\text{node}}(\tD^{(2),1}) = \mathbf{0}$, while $\pool^{\text{node}}(\tD^{(2),2}) \neq \mathbf{0}$. Therefore, when these vectors are passed to the downstream MPNN, it will be able to distinguish between the two graphs, completing the proof.
\end{proof}

\textbf {Computational complexity.}

We finish this section by proving Propositions \ref{prop:derivative_complexity} and \ref{prop:encoder_complexity},which analyze the time and space complexity of \ourmethod.

\begin{proof}
     
    Recall that in Appendix \ref{apx:derivative_comutation}, we have defined the derivative sparsity of a node $v$ of an input graph $\gG$ at layer $t$ of the base MPNN, as 
\begin{equation}
    s_{v,t} = \left|\left\{(\vu, i, \valpha)\;\middle|\; \tD^{(t)}[v, \vu,i, \valpha] \neq 0 \right\}\right|,
\end{equation}
and the maximal  derivative sparsity at layer $t$ is
\begin{equation}
    s_t = \min_{v \in V(\gG)} s_{v,t}.
\end{equation}
the tensor $\tD^{(t)}$ has memory complexity of  $O(n \cdot s_t)$ as it can be stored using sparse matrix representations. Algorithm \ref{alg:derivative_computation} for efficient derivative computation iteratively computes $\tD^{(t)}$ using $\tD^{(t-1)}$ and finally computes  $\tD^{\text{out}}$ from $\tD^{(T)}$. We now prove by induction that for an input graph $\gG$ with maximal degree $d$\footnote{We slightly abuse notation by using $d$ to denote the maximal degree as it usually denotes the input feature dimension, we denote this quantity as $l$ for this discussion.},  there exists a constant $C$ which depends on the maximal number of derivatives, the input node feature dimension, and the dimension of $\vh^{(t)}$, all of which are hyperparameters, such that
\begin{equation}
\label{eq:induction_complexity_proof}
    s_{t} < C \cdot \min \{ d^{k \cdot t}, n^k \}.
\end{equation}
We additionally show that $\tD^{(t)}$ can be constructed from $\tD^{(t-1)}$ in time complexity of $O(d \cdot n \cdot \min \{ d^{k \cdot (t-1)}, n^k \})$

To begin, as we saw in Appendix \ref{apx:derivative_comutation}, 
\begin{equation}
    \tD^{(0)}[v, \vu, i, \vj, \valpha] =
    \begin{cases}
        1 & \text{if } \exists s \text{ s.t. }  v = \vu_s, i = \vj_s, \valpha_s = 1,  \sum_{s' \neq s}\valpha_{s'} = 0\\
        0 & \text{otherwise},
    \end{cases}
\end{equation}
and so $s_{v,0} <  C = m^k \cdot l^{k+1} $ where $m$ is the maximal derivative degree, and $l$ is the dimension of the input node features. Now assuming Equation \ref{eq:induction_complexity_proof} holds for $t-1$, in Appendix \ref{apx:derivative_comutation} we saw that $\tD^{(t)}$ can be computed from $\tD^{(t-1)}$ in time complexity of $O(d\cdot n \cdot  s_{ t-1}) = O(d\cdot n \cdot  \min \{ d^{k \cdot (t-1)}, n^k \})$, secondly, we saw that  $s_{t} = O\left( \min\{d\cdot s_{t-1},\; n^k\}\right) = O\left(\min \{ d^{k \cdot t}, n^k \}\right)$ completing the induction. As the memory complexity of storing $\tD^{(t)}$ is $O(n \cdot s_t) = O\left(n \cdot \min \{ d^{k \cdot t}, n^k \}\right)$ this proves Proposition \ref{prop:derivative_complexity}. Recall now that we have seen in the proof of theorem \ref{thm:ourmethod_vs_osan_formal} that k-\ourmethod~ is able to achieve the same expressivity as $k$-OSAN by disregarding $\tD^{\text{out}}$ and choosing the encoder $\pool^{\text{node}}$ to be

\begin{equation}
    \pool^{\text{ds-node}}(\tD^{(T)})_v = \text{DeepSet}(\{\tD^{(T)}[v, \vu,i,  \valpha] \mid  \vu \in V^k(\gG), i \in [d'], \in [d]^k, \valpha \in [m]^k \}).
\end{equation}

This encoder can leverage the sparsity of $\tD^{(T)}$, having runtime complexity of $O(n \cdot s_T) = O( n \cdot \min\{n^k,\, d^{k \cdot T}\})$, completing the proof of proposition \ref{prop:encoder_complexity}.
 
\end{proof}

\section{Experimental Details}
\label{app:experimental_details}
In this section, we provide details on the experimental validation described and discussed in Section \ref{sec:experiments}.

\subsection{architecture}
In all of our experiments, we have used a 1-\ourmethod~ architecture, We now describe its components in detail.

\textbf{Base MPNN.}
In all experiments, we use a GIN \citep{xu2018powerful} architecture as the base MPNN—described in Equation~\ref{eq:gin_combined} (Section~\ref{sec:method})—due to its simplicity and maximal expressivity. All MLPs used in the node updates are two layers deep and employ ReLU activations. We extract first-order derivatives from the base MPNN, which—by Theorem~\ref{thm:ourmethod_vs_rwse}—yields greater expressivity than MPNNs augmented with RWSE. After computing the final node representations $\vh^{(T)}$, we apply a residual connection by aggregating all intermediate layers:
\begin{equation}
\vh^{(T)} \gets \bigoplus_{t=1}^T \frac{\vh^{(t)}}{t!}.
\end{equation}
where normalizing in $t!$ helped stabalize training.

\textbf{Derivative Encoding.}

For simplicity, we ignore the output derivatives of the base MPNN and use only the final node-wise derivative tensor $\tD^{(T)}$. The encoder $\pool^{\text{node}}$ is implemented as a lightweight, efficient module that applies a pointwise MLP to the diagonal entries of $\tD^{(T)}$:
\begin{equation}
\pool^{\text{node}}(\tD^{(T)})_v = \mlp\left(\tD^{(T)}[v, v, \dots]\right).
\end{equation}

\textbf{Downstream GNN.} 

For simplicity and to isolate the effects of \ourmethod, we restrict our experiments to MPNNs as downstream architectures—We use GIN for all experiments but peptieds func in which we used a GCN \citep{kipf2016semi}. All MLPs used in the node updates are two layers deep with ReLU activations. The final MLP head also uses ReLU activations and consists of 1 to 3 layers, following the default settings from \citet{southern2025balancing}, without tuning. In experiments with parameter budgets, we adjust only the hidden dimension to fit within the limit, selecting the largest value that satisfies the constraint.

\textbf{Initialization.}
We initialize the base MPNN such that all categorical feature embeddings are set to the constant vector $\mathbf{1}$, all MLP weights are initialized to the identity, and the $\epsilon$ parameters in the GIN update (Equation~\ref{eq:gin_combined}) are set to $-1$. This initialization is motivated by the proof of Theorem~\ref{thm:ourmethod_vs_rwse}; following a similar line of reasoning, it implies that the diagonal of the derivative tensor $\tD^{(T)}[v,v,\dots]$ corresponds exactly to the centrality encoding proposed in \citet{southern2025balancing}.

\textbf{Optimization.} We use separate learning rates for the base and downstream MPNNs. For the downstream MPNN, we adopt the learning rates used in \citet{southern2025balancing}, while for the base MPNN, we fine-tune by selecting either the same rate or one-tenth of it.

\subsection{Experiments}

We provide below the details of the datasets and hyperparameter configurations used in our experiments. Our method is implemented using PyTorch~\citep{paszke2019pytorch} and PyTorch Geometric~\citep{fey2019fast}, and is based on code provided in \citet{southern2025balancing} and \citet{rampavsek2022recipe}. Test performance is evaluated at the epoch achieving the best validation score and is averaged over four runs with different random seeds. We optimize all models using AdamW~\citep{loshchilov2017decoupled}, with a linear learning rate warm-up followed by cosine decay. We track experiments and perform hyperparameter optimization using the Weights and Biases platform.  All experiments were conducted on a single NVIDIA A100-SXM4-40GB GPU.

\textbf{OGB datasets.} We evaluate on three molecular property prediction benchmarks from the OGB suite~\citep{hu2020open}:  MOLHIV, MOLBACE, and MOLTOX21. These datasets share a standardized node and edge featurization capturing chemophysical properties. We adopt the challenging scaffold split proposed in~\citep{hu2020graph}. To prevent memory issues, we use a batch size of 128 for MOLHIV and 32 for the remaining datasets. All Downstream models use a hidden dimension of 300, consistent with prior work~\citep{hu2020graph, bevilacqua2023efficient}. We sweep over several architectural choices, including the hidden dimension of the Base MPNN $k  = 8, 10, \dots, 30, 32$, the initial learining rate of the Base MPNN $k  = 0.001, 0.0001$, and the dropout rate $k=0.0, 0.1, 0.2, 0.3, 0.4, 0.5$ . The number of layers in the base MPNN was selected to match the positional encoding step used in the corresponding experiment from \citet{southern2025balancing}. Hyperparameter tuning was performed on the validation set using four random seeds. Results are reported at the test epoch corresponding to the best validation performance. All models were trained for 100 epochs. The final parameters used for each experiment are reported in table \ref{tab:best_hparams_ogb}

% \begin{table}[h]
% \centering
% \caption{Best-performing hyperparameters for each OGB dataset.}
% \label{tab:best_hparams_ogb}
% \begin{tabular}{lccc}
% \toprule
% \textbf{Hyperparameter} & \textbf{MOLHIV} & \textbf{MOLBACE} & \textbf{MOLTOX21} \\
% \midrule
% \multicolumn{4}{l}{\textbf{Downstream model}} \\
% \#Layers & 2 & 8 & 10 \\
% \#Readout Layers & 1 & 3 & 3 \\
% Hidden Dimension & 300 & 300 & 300 \\
% Dropout & 0.0 & 0.5 & 0.3 \\
% Learning Rate & 0.0001 & 0.0001 & 0.001 \\
% \midrule
% \multicolumn{4}{l}{\textbf{Base MPNN}} \\
% \#Layers & 16 & 20 & 20 \\
% Hidden Dimension & 16 & 16 & 16 \\
% Dropout & 0.2 & 0.5 & 0.2 \\
% Learning Rate & 0.0001 & 0.0001 & 0.0001 \\
% \midrule
% \#Parameters & 419,403 & 1,691,329 & 2,061,322 \\
% \bottomrule
% \end{tabular}
% \end{table}

\textbf{Zinc} The ZINC dataset~\citep{dwivedi2023benchmarking} includes 12k molecular graphs of commercially available chemical compounds, with the task of predicting molecular solubility. We follow the predefined dataset splits and report the Mean Absolute Error (MAE) as both the loss and evaluation metric. Our downstream MPNN for this task includes 6 message-passing layers and 3 readout layers, with a hidden size of 120 and no dropout. We use a batch size of 32 and train for 2000 epochs. We performed a small sweep over the depth of the base MNPNN $k={10,12,14,16,18,20}$ and the hidden dimension $k={30,35,\dots,80}$. The hidden dimension of the downstream MPNN was chose as $120$ to meet the $500k$ parameter constraint. The final hyperparameters are listed in Table~\ref{tab:best_hparams_zinc}.

% \begin{table}[h]
% \centering
% \caption{Best-performing hyperparameters for the ZINC dataset.}
% \label{tab:best_hparams_zinc}
% \begin{tabular}{lc}
% \toprule
% \textbf{Hyperparameter} & \textbf{ZINC} \\
% \midrule
% \multicolumn{2}{l}{\textbf{Downstream model}} \\
% \#Layers &  8 \\
% \#Readout Layers & 3 \\
% Hidden Dimension & 120 \\
% Dropout & 0.0 \\
% Learning Rate & 0.001\\
% \midrule
% \multicolumn{2}{l}{\textbf{Base MPNN}} \\
% \#Layers & 10 \\
% Hidden Dimension & 64 \\
% Dropout & 0.0 \\
% Learning Rate & 0.0001\\
% \midrule
% \#Parameters & 497,353 \\
% \bottomrule
% \end{tabular}
% \end{table}

\textbf{Peptides} Peptides-func and Peptides-struct, introduced by \citet{dwivedi2022long}, consist of graphs representing atomic peptides. Peptides-func is a multi-label classification benchmark with 10 nonexclusive peptide function labels, while Peptides-struct is a regression task involving 11 different structural attributes derived from 3D conformations.

For both datasets, we adopt the hyperparameter setup proposed by \citet{tonshoff2023did} for the downstream GNN, which has a parameter budget under 500k and where they use 250 epochs. We set the number of message-passing layers in our base MPNN with the positional encoding steps to be 20, aligned with the number of steps used for the random-walk structural encoding.  The only tuned component is the learning rate of the base MPNN. The final configurations are summarized in Table~\ref{tab:best_hparams_peptides}.

\textbf{Key empirical findings.} Across all benchmarks, \ourmethod\ is highly competitive and \emph{the only architecture that consistently ranks within the top two model tiers}. Additionally, the strong performance of \ourmethod\ on the large-scale Peptides datasets—where full-bag Subgraph GNNs are generally unable to run—shows its ability to scale effectively. Notably, the base MPNNs in \ourmethod\ are often significantly deeper and narrower than those in typical GNNs. For instance, on the OGB datasets, the base MPNNs use 17-20 layers with hidden dimensions as low as 16--32. Despite their compact size, these base MPNNs yield notable performance gains over standard GINE, suggesting potential robustness to oversquashing. Moreover, the fact that they are significantly deeper than typical MPNNs further suggests that \ourmethod\ may help mitigate oversmoothing, a hypothesis we leave for future work.

\vspace{-10pt} 

\begin{table}[h]
\centering
\caption{Best-performing hyperparameters for each OGB dataset.}
\label{tab:best_hparams_ogb}
\begin{tabular}{lccc}
\toprule
\textbf{Hyperparameter} & \textbf{MOLHIV} & \textbf{MOLBACE} & \textbf{MOLTOX21} \\
\midrule
\multicolumn{4}{l}{\textbf{Downstream model}} \\
\#Layers & 2 & 8 & 10 \\
\#Readout Layers & 1 & 3 & 3 \\
Hidden Dimension & 300 & 300 & 300 \\
Dropout & 0.0 & 0.5 & 0.3 \\
Learning Rate & 0.0001 & 0.0001 & 0.001 \\
\midrule
\multicolumn{4}{l}{\textbf{Base MPNN}} \\
\#Layers & 16 & 20 & 20 \\
Hidden Dimension & 16 & 16 & 16 \\
Dropout & 0.2 & 0.5 & 0.2 \\
Learning Rate & 0.0001 & 0.0001 & 0.0001 \\
\midrule
\#Parameters & 450,848 & 1,723,448 & 2,165,782 \\
\bottomrule
\end{tabular}
\end{table}

\begin{table}[h]
\centering
\caption{Best-performing hyperparameters for the ZINC dataset.}
\label{tab:best_hparams_zinc}
\begin{tabular}{lc}
\toprule
\textbf{Hyperparameter} & \textbf{ZINC} \\
\midrule
\multicolumn{2}{l}{\textbf{Downstream model}} \\
\#Layers &  6 \\
\#Readout Layers & 3 \\
Hidden Dimension & 120 \\
Dropout & 0.0 \\
Learning Rate & 0.001\\
\midrule
\multicolumn{2}{l}{\textbf{Base MPNN}} \\
\#Layers & 12 \\
Hidden Dimension & 75 \\
Dropout & 0.0 \\
Learning Rate & 0.0001\\
\midrule
\#Parameters & 498,144 \\
\bottomrule
\end{tabular}
\end{table}

\begin{table}[h]
\centering
\caption{Best-performing hyperparameters for the Peptides-func and Peptides-struct datasets.}
\label{tab:best_hparams_peptides}
\begin{tabular}{lcc}
\toprule
\textbf{Hyperparameter} & \textbf{Peptides-func} & \textbf{Peptides-struct} \\
\midrule
\multicolumn{3}{l}{\textbf{Downstream model}} \\
\#Layers & 6 & 10 \\
\#Readout Layers & 3 & 3 \\
Hidden Dimension & 234 & 143 \\
Dropout & 0.1 & 0.2 \\
Learning Rate & 0.001 & 0.001 \\
\midrule
\multicolumn{3}{l}{\textbf{Base MPNN}} \\
\#Layers & 20 & 20 \\
Hidden Dimension & 8 & 8 \\
Dropout & 0.1 & 0.2 \\
Learning Rate & 0.0001 & 0.001 \\
\midrule
\#Parameters & 498,806 & 493,849 \\
\bottomrule
\end{tabular}

\end{table}

\section{Additional Experiments}
\label{app:additional_experiments}

\subsection{Substructure Counting}

\begin{table}[t]
\setlength{\tabcolsep}{4pt}
\centering
\footnotesize
\vspace{-6pt}
\caption{Normalized MAE results on the counting subgraphs dataset. 
Cells below 0.01 are highlighted in yellow.}
\begin{tabular}{@{\hspace{4pt}}l | c c c c c c c c}
\toprule
\textbf{Method} & 3-Cycle & 4-Cycle & 5-Cycle & 6-Cycle & Tailed Tri. & Chordal Cycle & 4-Clique & 4-Path \\
\midrule

MPNN         & 0.3515 & 0.2742 & 0.2088 & 0.1555 & 0.3631 & 0.3114 & 0.1645 & 0.1592 \\

MPNN+RWSE    & 0.0645 & 0.0264 & 0.0746 & 0.0578 & 0.0505 & 0.1008 & 0.0905 & 0.0217 \\

GPS+RWSE   & 0.0185 & 0.0433 & 0.0472 & 0.0551  & 0.0446 &0.0974 &  0.0836 & 0.0284 \\

HyMN         & 0.0384 & 0.0933 & 0.1350 & 0.0936 & \cellcolor{yellow!25}0.0084 & 0.0746 & 0.0680 & 0.0120 \\

GNN-AK+      & \cellcolor{yellow!25}0.0004 & \cellcolor{yellow!25}0.0040 & 0.0133 & 0.0238 & \cellcolor{yellow!25}0.0043 & 0.0112 & \cellcolor{yellow!25}0.0049 & \cellcolor{yellow!25}0.0075 \\

\ourmethod + ReLU & \cellcolor{yellow!25}0.0012 & \cellcolor{yellow!25}0.0046 & 0.0210 & 0.0380 & \cellcolor{yellow!25} 0.0083 & 0.0510 & 0.0293 & \cellcolor{yellow!25}0.0081 \\

\ourmethod + SiLU   & \cellcolor{yellow!25}0.0008 & \cellcolor{yellow!25}0.0042 & \cellcolor{yellow!25}0.0068 & 0.0222 & \cellcolor{yellow!25}0.0066 & 0.0195 & \cellcolor{yellow!25}0.0055 & \cellcolor{yellow!25}0.0069 \\

\bottomrule
\end{tabular}
\label{tab:subgraph_counting}
\end{table}

We adopt the synthetic node-level subgraph counting experiment used in \citet{huang2022boosting, yan2024efficient}. The dataset consists of 5,000 graphs generated from a mixture of distributions (see \citet{zhao2022from} for more details), with a train/validation/test split of 0.3/0.2/0.5. The task is node-level regression: predicting the number of substructures such as 3-cycles, 4-cycles, 5-cycles, 6-cycles, tailed triangles, chordal cycles, 4-cliques and 4-paths, where continuous outputs approximate discrete counts. We report the normalized MAE for each baseline, and highlight cases where the error falls below 0.001, since in these instances rounding the predictions yields exact counts.  

For training, we use the AdamW optimizer with an initial learning rate of 0.001, a cosine scheduler with warmup, and train for 5,00 epochs. The batch size is set to 128. We compare \ourmethod against both positional/structural encoding methods (MPNN+RWSE \citep{dwivedi2021graph}, GPS+RWSE \citep{rampavsek2022recipe}, HyMN \citep{southern2025balancing}) and node-based subgraph GNNs (GNN-AK+ \citep{zhao2022from}, Nested GNN \citep{zhang2021nested}, ID-GNN \citep{you2021identity}, HyMN \citep{southern2025balancing}).  

To examine the role of activation functions, we evaluate \ourmethod with the non-analytic ReLU and the analytic SiLU. The results, summarized in Table~\ref{tab:subgraph_counting}, are consistent with our theory. In line with Theorem~\ref{thm:ourmethod_vs_osan}, \ourmethod achieves performance comparable to or surpassing other subgraph GNNs. Furthermore, consistent with Theorem~\ref{thm:ourmethod_vs_rwse}, we observe that analytic activations enhance expressivity, while even with non-analytic ReLU, \ourmethod still outperforms encoding-based methods.

\subsection{Ablation on Derivative Order}
To assess the impact of higher-order derivatives on \ourmethod's expressive power, we conduct an ablation study on the most challenging substructure counting task: 6-cycle prediction, which consistently yields the highest MAE. Specifically, we evaluate models restricted to derivatives of order at most $k$, for $k = 0,\ldots,4$. The results, presented in Table~\ref{tab:derivative_ablation}, demonstrate that increasing the maximal derivative order consistently improves MAE, with performance saturating at $k=4$. All experiments use the analytic SiLU activation function. We hypothesize that the observed saturation arises because higher-order derivatives of SiLU rapidly diminish toward zero, limiting the additional expressive gain.

\begin{table}[H]
\centering
\footnotesize
\caption{Effect of maximal derivative order on 6-Cycle MAE.}
\begin{tabular}{c c}
\toprule
\textbf{Maximal derivative order} & \textbf{6-Cycle MAE} \\
\midrule
0 & 0.1555 \\
1 & 0.0275 \\
2 & 0.0223 \\
3 & 0.0221 \\
4 & 0.0231 \\
\bottomrule
\end{tabular}
\label{tab:derivative_ablation}
\end{table}

\subsection{Stability Analysis }

We evaluate the stability of \ourmethod~both in terms of training dynamics and the behavior of the derivative tensor norm. Across all tasks, we consistently observe smooth and stable loss curves, even when incorporating higher-order derivatives. To further assess stability, we conduct two dedicated experiments.

\paragraph{Training Loss Dynamics.}
We examine convergence behavior on the 6-cycle subgraph counting task by varying the maximum derivative order $d \in {1,2,3,4}$. Figure ~\ref{fig:loss_curves}  shows the training loss curves, showing consistent and reliable convergence across all derivative orders.

\begin{figure*}[t]
    \centering
    \includegraphics[width=0.73\textwidth]{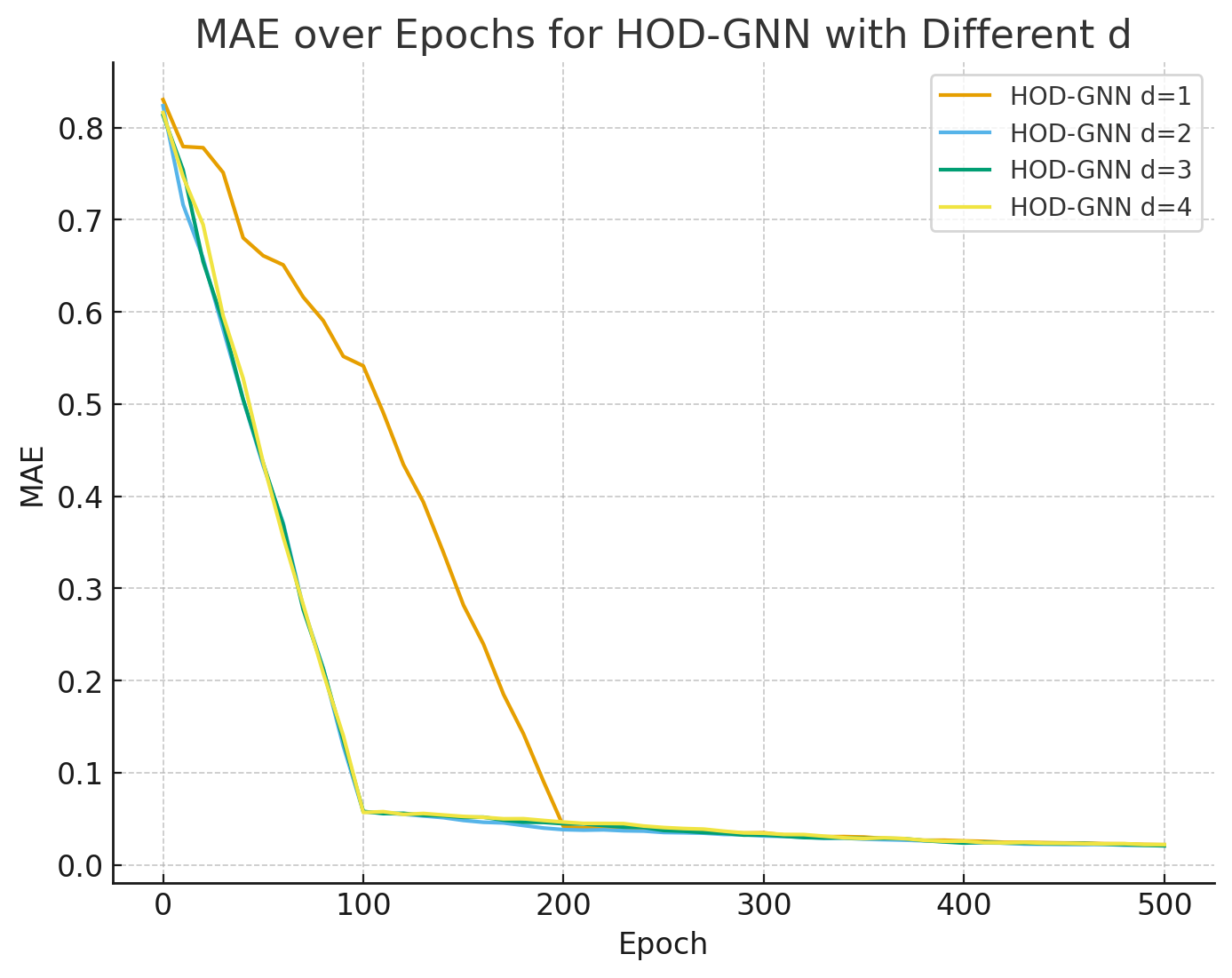}
    \caption{The training loss of \ourmethod with maximal derivative $d \in \{1,2,3,4\}$ on the 6-cycle counting task. }\label{fig:loss_curves}
\end{figure*}

\begin{table}[t]
\centering
\caption{Runtime and memory comparison on the MOLHIV dataset. HOD-GNN demonstrates both improved memory efficiency and competitive runtime.}

\begin{tabular}{l|c|c|c}
\toprule
GNN & GPU Memory (MiB) & Training Time / Epoch (s) & Test Time / Epoch (s) \\
\midrule
edge-deletion & 32,944 & 62.13 & 5.70 \\
node-deletion & 29,826 & 58.80 & 4.01 \\
ego-nets      & 25,104 & \textbf{53.16} & 3.07 \\
ego-nets+     & 25,211 & 54.19 & 3.20 \\
\ourmethod    & \textbf{12,964} & 53.34 & \textbf{2.28} \\
\bottomrule
\end{tabular}
\label{tab:scalability}
\end{table}

\paragraph{Norm of the Derivative Tensor.}
We additionally assess the stability of the derivative tensors used by \ourmethod.  We compute the norm of the derivative tensor during training on the MOLBACE dataset. Following standard practice~\citep{higham2002accuracy}, we report the relative norm: the ratio between the derivative tensor norm and the final node feature norm of the base MPNN. Figure~\ref{fig:derivative_norm} shows that the relative norm  consistently remains significantly lower than that of the node features throughout training, confirming that \ourmethod~operates in a well-conditioned regime.

\begin{figure*}[t]
    \centering
    \includegraphics[width=0.73\textwidth]{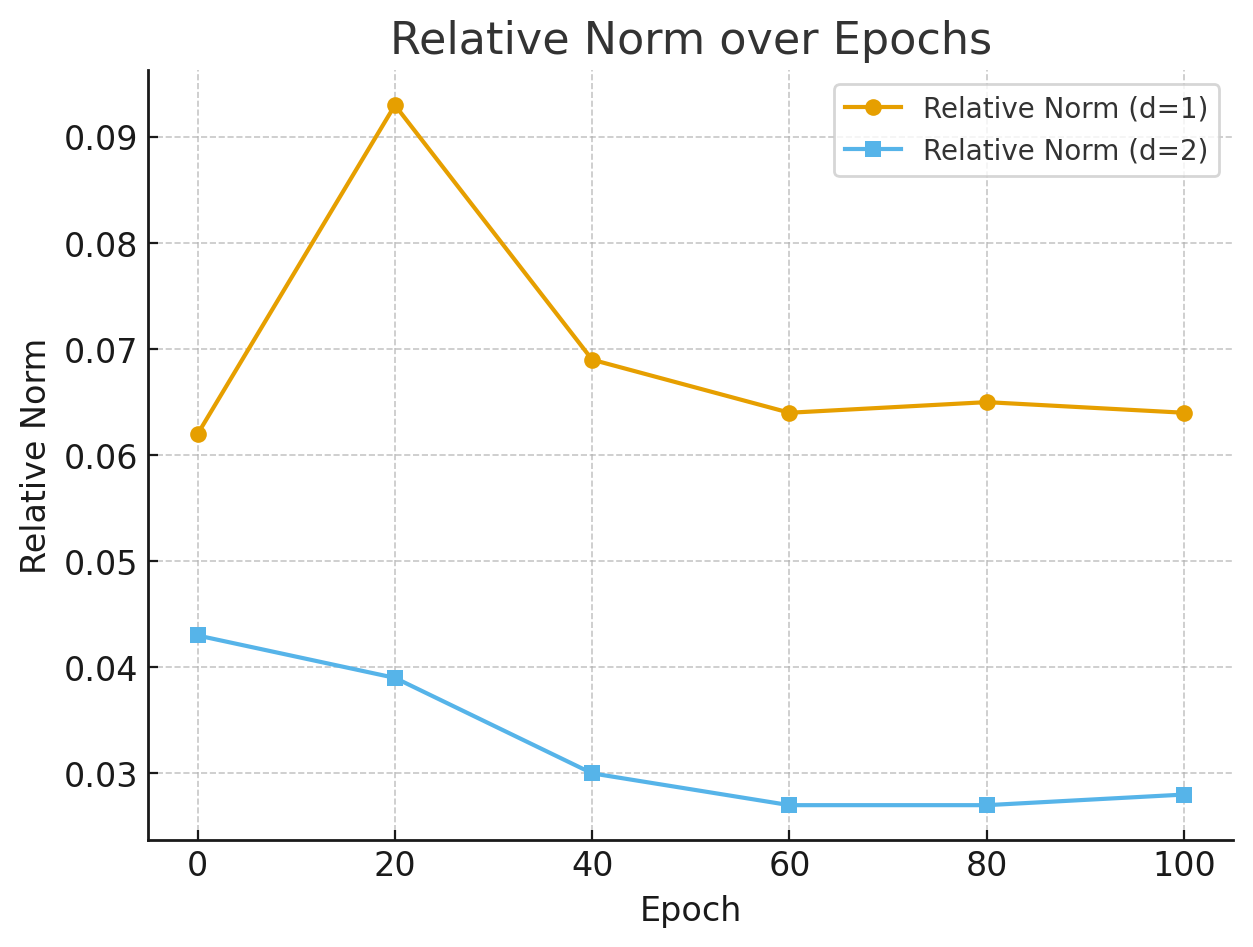}
    \caption{The relative norm of the derivative tensor with maximal derivative $d \in \{1,2\}$ with respect to the norm of the final node feature matrix of the base MPNN on the MOLBACE dataset. }\label{fig:derivative_norm}
\end{figure*}

\subsection{Comparison of Runtime and Memory usage against subgraph GNNs}

In section \ref{sec:experiments} we show that \ourmethod is able to scale to the Peptides datasets, which are unreachable for full-bag subgraph GNNs on standard hardware, as stated in \citet{southern2025balancing, bar-shalom2023subgraphormer}.To further demonstrate the scalability of \ourmethod, we benchmark its runtime and memory usage against subgraph-based GNNs on the MOLHIV dataset. Since the choice of subgraph selection policy is the primary factor determining the asymptotic complexity of subgraph GNNs ~\citep{bevilacqua2021equivariant}, we evaluate a range of policies using the default parameters from the original work. For subgraph GNNs, we adopt the hyperparameters reported in \citet{bevilacqua2021equivariant} for MOLHIV. For \ourmethod, we use the same hyperparameters as in Section~\ref{sec:experiments}, detailed in Appendix~\ref{app:experimental_details}.

\begin{table}[t]
\centering
\caption{MAE comparison of different GNN architectures on ZINC-12K.}
\begin{tabular}{l c}
\toprule
\textbf{GNN} & \textbf{MAE} \\
\midrule

GCN              & 0.321 $\pm$ 0.009 \\
HOD-GNN + GCN    & 0.080 $\pm$ 0.006 \\
\midrule

GIN              & 0.163 $\pm$ 0.004 \\
HOD-GNN + GIN    & 0.066 $\pm$ 0.003 \\
\midrule

GPS              & 0.070 $\pm$ 0.004 \\
HOD-GNN + GPS    & 0.064 $\pm$ 0.002 \\
\bottomrule
\end{tabular}
\label{tab:backbone}
\end{table}

\paragraph{Results.}
Table~\ref{tab:scalability} reports GPU memory usage and per-epoch training and test runtimes. \ourmethod~achieves improvements in memory efficiency, requiring less than half the GPU memory compared to subgraph GNNs. In terms of runtime, HOD-GNN is faster than edge-deletion and node-deletion policies, while achieving comparable training time to ego-nets and ego-nets+. Notably, ego-net policies are known to be less expressive ~\citep{bevilacqua2021equivariant}, highlighting that HOD-GNN achieves both efficiency and expressivity.

% ~\citep{barcelo2020logical}

\paragraph{Discussion.}
The improvements primarily stem from the analytic computation of higher-order derivatives in \ourmethod, which avoids the costly enumeration of subgraphs. While these advantages already translate to lower memory usage and faster runtimes in practice, we emphasize that HOD-GNN’s scalability potential is not yet fully realized. In particular, it relies on efficient sparse matrix multiplications, which are currently suboptimally implemented in popular GNN libraries such as PyTorch Geometric. We therefore anticipate that further optimization of sparse kernels would amplify the scalability benefits of HOD-GNN.

\subsection{Ablation on Backbone GNN}
In most experiments, we employ GIN~\citep{xu2018powerful} as the backbone of \ourmethod. A key advantage of our approach, however, is its compatibility with any message-passing backbone. To assess the effect of backbone choice, we evaluate \ourmethod~on the \textsc{Zinc} dataset using GCN, GIN, and GPS as base architectures.
The results, summarized in Table~\ref{tab:backbone}, demonstrate that \ourmethod~consistently improves upon its backbone across all settings. Among the tested architectures, GPS achieves the strongest performance, followed by GIN, with GCN ranking third—reflecting its lower expressivity relative to the other backbones.

\end{document}